\documentclass{article}

\usepackage[preprint]{neurips_2025}

\usepackage[utf8]{inputenc} 
\usepackage[T1]{fontenc}    
\usepackage{hyperref}       
\usepackage{url}            
\usepackage{booktabs}       
\usepackage{amsfonts}       
\usepackage{nicefrac}       
\usepackage{microtype}      
\usepackage{xcolor}         

\usepackage[american]{babel}

\usepackage{mathtools} 
\usepackage{booktabs} 
\usepackage{tikz} 

\usepackage{algorithm}
\usepackage{algorithmic}

\usepackage{microtype}
\usepackage{graphicx}
\usepackage{booktabs}

\usepackage{thmtools,thm-restate}

\usepackage{hyperref}
\usepackage{caption}
\usepackage{subcaption}
\usepackage{bbm}

\usepackage{amsmath}
\usepackage{amssymb}
\usepackage{mathtools}
\usepackage{amsthm}

\usepackage{mwe}
\usepackage{tikz} 
\usepackage{pgf}
\usepackage{bbm}
\usepackage{extarrows}

\usepackage[capitalize,noabbrev]{cleveref}

\newcommand{\A}[0]{\mathcal{A}}
\newcommand{\V}[0]{\mathcal{V}}

\newcommand{\E}[0]{\mathcal{E}}

\newcommand{\pr}[0]{\mathbb{P}}

\newcommand{\G}[0]{\mathcal{G}}
\newcommand{\D}[0]{\mathcal{D}}

\newcommand{\kl}[1]{\text{KL}\left( #1 \right)}

\newcommand{\bmu}[0]{\boldsymbol{\mu}}
\newcommand{\bmuh}[0]{\hat{\bmu}}
\newcommand{\muh}[0]{\hat{\mu}}

\newcommand{\rh}[0]{\hat{\mathbf{R}}}

\newcommand{\rhi}[1]{\hat{\mathbf{R}}_i(#1)}
\newcommand{\rhj}[1]{\hat{\mathbf{R}}_j(#1)}
\newcommand{\rhdi}[1]{\hat{{R}}_i^d(#1)}
\newcommand{\rhdj}[1]{\hat{{R}}_j^d(#1)}
\newcommand{\dr}[0]{\text{d}(r)}

\newcommand{\br}[0]{\mathbf{R}}
\newcommand{\bw}[0]{\mathbf{w}}
\newcommand{\bn}[0]{\mathbf{n}}

\newcommand{\m}[2]{m \left( #1, #2 \right)}
\newcommand{\M}[2]{M \left( #1, #2 \right)}
\newcommand{\Meps}[2]{M_{\epsilon} \left( #1, #2 \right)}
\newcommand{\bet}[1]{\beta(#1, \delta)}
\newcommand{\betnd}[1]{\bet{n^d(#1)}}
\newcommand{\wlone}[0]{\lVert \bw \rVert_{1}}

\newcommand{\sumij}[0]{\sum_{\substack{i \in [N] \\ j \in [K]}}}

\newcommand{\istarmu}[0]{i^*(\bmu)}

\DeclareMathOperator*{\argmax}{arg\,max}

\theoremstyle{plain}
\newtheorem{theorem}{Theorem}[section]

\newtheorem{lemma}[theorem]{Lemma}

\theoremstyle{definition}
\newtheorem{definition}[theorem]{Definition}

\theoremstyle{remark}
\newtheorem{remark}[theorem]{Remark}

\usepackage[textsize=tiny]{todonotes}

\title{Best Group Identification in Multi-Objective Bandits}

\author{%
  Mohammad Shahverdikondori \\
  College of Management of Technology\\
  EPFL, Lausanne, Switzerland \\
  \texttt{mohammad.shahverdikondori@epfl.ch} \\
  \AND
  Mohammad Reza Badri \\
  Department of Computer Engineering\\
  Sharif University of Technology, Tehran, Iran\\
  \texttt{mreza.badri@sharif.edu} \\
  \And
  Negar Kiyavash \\
  College of Management of Technology \\
  EPFL, Lausanne, Switzerland \\
  \texttt{negar.kiyavash@epfl.ch} \\
}

\begin{document}

\maketitle

\begin{abstract}
    We introduce the \textit{Best Group Identification} problem in a multi-objective multi-armed bandit setting, where an agent interacts with groups of arms with vector-valued rewards. The performance of a group is determined by an \textit{efficiency} vector which represents the group's best attainable rewards across different dimensions. The objective is to identify the set of \textit{optimal} groups in the fixed-confidence setting.
    We investigate two key formulations: \textit{group Pareto set identification}, where efficiency vectors of optimal groups are Pareto optimal and \textit{linear best group identification}, where each reward dimension has a known weight and the optimal group maximizes the weighted sum of its efficiency vector’s entries. For both settings, we propose elimination-based algorithms, establish upper bounds on their sample complexity, and derive lower bounds that apply to any correct algorithm. Through numerical experiments, we demonstrate the strong empirical performance of the proposed algorithms.
\end{abstract}

\section{Introduction} \label{sec: intro}

The multi-armed bandit (MAB) framework has been extensively applied to a wide range of problems in stochastic sequential decision-making \cite{mab1-bouneffouf2020survey, bandit-book1-lattimore2020bandit, bandit-book2-bubeck2012regret, mab2-tewari2017ads}. In this framework, an agent interacts with an environment in a sequential manner, selecting an action (arm) at each round and receiving a random reward in return. The primary objectives in MAB problems typically fall into two categories: maximizing cumulative rewards over time (\textit{regret minimization}) or identifying certain properties of the reward distributions as efficiently as possible (\textit{pure exploration}).  
While pure exploration problems in single-objective bandits have been widely studied \cite{pure-bubeck2009pure, bai-audibert2010best, SH-karnin2013almost, kaufmann2021mixture, shahverdikondori2025optimal}, the multi-objective setting, where each arm's reward is vector-valued, remains under-explored. In this work, we investigate a pure exploration problem in a multi-objective MAB setting, which presents new challenges due to the need to evaluate trade-offs across multiple reward dimensions.  

We introduce and study the \textit{best group identification (BGI)} problem in multi-objective bandits. In this problem, the set of arms is partitioned into several groups, and each group is characterized by an \textit{efficiency} vector, which represents the maximum mean rewards attainable from the arms within the group across different reward dimensions. The objective is to identify the set of \textit{optimal} groups given a predetermined error probability while minimizing the number of rounds, a setting which is referred to as \textit{fixed-confidence}.  
The definition of optimality depends on the agent’s decision-making goal. We consider \emph{group Pareto set identification (GPSI)}, where the objective is to identify the set of groups whose efficiency vectors are Pareto optimal, meaning that no other group uniformly outperforms them across all objectives. This problem generalizes the problem of Pareto set identification (PSI)  in multi-objective bandits \cite{psi-auer2016pareto, psi-budget-kone2024bandit}. Additionally, we consider \emph{linear best group identification}, where each reward dimension is assigned a known weight, and the goal is to identify the group with the highest weighted reward, computed using the group's efficiency vector. For a detailed discussion on the related work, refer to Appendix \ref{apx: related-work}.

As a motivating example, consider a marketing scenario where, a company organizes its offerings into bundles, each containing a variety of products. The best possible performance of each bundle is measured across multiple metrics such as profit, customer satisfaction, and product return rates, etc. The company can sequentially research individual products to evaluate their performance, with the goal of efficiently determining which product bundles should be focused on. This process naturally falls into the best group identification setting we study.

\subsection{Contributions} 
Our main contributions are as follows. 

\begin{itemize}  
    \item We formally introduce a general pure exploration setting in multi-objective MABs, which we shall call best group identification (BGI), where the objective is to identify the optimal set of groups based on their efficiency vectors while minimizing the sample complexity.
    
    \item As a key instance of BGI, we study group Pareto set identification (GPSI), where the goal is to identify the set of Pareto optimal groups. We propose the Triple Elimination (TE) algorithm for this problem and establish an instance-dependent upper bound on its sample complexity. Additionally, we derive a lower bound on the sample complexity of any correct algorithm, demonstrating that TE is near-optimal in some classes of problem instances.  
    
    \item We further investigate the Linear Best Group Identification problem (LBGI), in which each dimension of the reward vector has a known weight, and the objective is to identify the group with the highest weighted reward. In this setting, we introduce the Equal Effect Confidence Bound (EECB) algorithm and provide an instance-dependent sample complexity upper bound. We also establish a lower bound on the sample complexity of any algorithm that solves this problem.  
      
    \item Through numerical simulations, we validate the effectiveness of the proposed algorithms in different settings.
\end{itemize}

\paragraph{Paper Outline.} In Section \ref{sec: setup}, we present the formal definition of the BGI problem along with the necessary background. The results for the GPSI and  LBGI problems are presented in Sections \ref{sec: gpsi} and \ref{sec: lbgi}, respectively. Our numerical experiments are discussed in Section \ref{sec: exp}. A more detailed discussion of the related work, along with all the proofs, is deferred to the Appendix.

\section{Problem Setup} \label{sec: setup}

In this section, we introduce the \textit{best group identification (BGI)} problem, where a player interacts with a multi-objective multi-bandit environment $\V$ consisting of $N$ bandits $\G_1, \G_2, \ldots, \G_N$, referred to as groups. Each group $\G_i$ has $K$ arms \footnote{The assumption of an equal number of arms per group is made for simplicity. Extending this framework to allow different numbers of arms per group is straightforward.}. We denote the set of arms by $\A \triangleq [N] \times [K] \triangleq \{(i,j) \mid i \in \{1,2,\ldots,N\}, j \in \{1,2,\ldots,K\}\}$. 

At each round $t$, the player selects an arm $(i_t, j_t)$, which corresponds to pulling the $j_t$-th arm in group $\G_{i_t}$, and observes a sample $X_t$ from a $D$-dimensional stochastic reward vector with expectation $\mathbb{E}[X_t] = \bmu_{i_t,j_t} \in \mathbb{R}^D$.
This environment includes an $N \times K \times D$ tensor $\bmu$ called arm means tensor, where $\mu_{i,j}^d$ denotes the $d$-th component of the mean reward vector corresponding to arm $(i,j)$. Let $\mathbf{v} \in \mathbb{R}^D$ and $u \in \mathbb{R}$, for any $d \in [D]$, we use super script $d$ to show the $d$-th component of $\mathbf{v}$ and $\mathbf{v} + u \triangleq  (v^1 + u, v^2 + u, \ldots, v^D + u)$. 

For each group $\G_i$, we define the \textit{efficiency} vector $\br_i \in \mathbb{R}^D$ with $R_i^d = \max_{j \in [K]} \mu_{i,j}^d$ which contains the maximum expected reward for each reward component among the arms in $\G_i$. Define $\br \triangleq  (\br_1, \br_2, \ldots, \br_N)$.

\begin{definition}[Domination]\label{def: domination}
Given two vectors $\mathbf{u}, \mathbf{v} \in \mathbb{R^D}$, we say that $\mathbf{u}$ is weakly dominated by $\mathbf{v}$ (denoted by $\mathbf{u} \leq \mathbf{v}$) if $\forall d \in [D]: u^d \leq v^d$,
$\mathbf{u}$ is dominated by $\mathbf{v}$ ($\mathbf{u} \preceq \mathbf{v}$) if $\mathbf{u}$ is weakly dominated by $\mathbf{v}$ and there exists a $d \in [D]$ with $u^d < v^d$,
and $\mathbf{u}$ is strictly dominated by $\mathbf{v}$ ($\mathbf{u} \prec \mathbf{v}$) if $\forall d \in [D]: u^d < v^d$. We also say an arm (group) is dominated by another arm (group) if its reward (efficiency) vector is dominated by the reward (efficiency) vector of the other one.
\end{definition} 

Following the definition of the efficiency vector for each group, we can compare the groups based on various criteria and develop algorithms to identify the set of optimal groups.
The set $i^*(\bmu) \subseteq [N]$ of group indices that represents the optimal groups thus varies based on the chosen criterion.
~A sequential learning algorithm for solving the BGI problem aims to identify this set with high probability within a minimum number of rounds. Such an algorithm requires three components: A sampling rule $P_t$ (deterministic or stochastic), which selects the arm to be pulled in round $t$ based on the observations up to round $t-1$, denoted by $\mathcal{H}_{t-1} = ((i_1, j_1), X_1, (i_2, j_2), X_2, \ldots, (i_{t-1}, j_{t-1}), X_{t-1})$, a stopping rule $\tau$, which, based on $\mathcal{H}_t$, determines when to stop and recommend the set of optimal groups, and a final recommendation $\hat{i}_{\tau}$ for $i^*(\bmu)$.
We denote such an algorithm by $(P_t, \tau, \hat{i}_{\tau})$. 

In this paper, we assume that the output distribution of each arm is sub-Gaussian with parameter $1$ and that the reward means are bounded within the interval $[0,1]$. This is a standard assumption in the bandit literature \cite{bandit-book1-lattimore2020bandit}.

\section{Group Pareto Set Identification Problem} \label{sec: gpsi}

In this section, we first provide the necessary definitions and properties to formally define the \textit{group Pareto set identification (GPSI)} problem. Subsequently, we propose an algorithm and analyze its sample complexity.
Finally, we present a lower bound on the number of samples required for any algorithm that solves the GPSI problem.
\begin{definition} [$\epsilon$-Pareto] \label{def: epsilon-pareto}
    For any constant $\epsilon \geq 0$, the $\epsilon$-Pareto set of groups $\G^*_{\epsilon}(\bmu)$ is defined as 
    \begin{align*}
        \G^*_{\epsilon}(\bmu) \triangleq \{i \in [N] \big| \nexists j \in [N]:  \br_i + \epsilon \prec \br_j \}.
    \end{align*}
    This set is an $\epsilon$-approximation of the Pareto optimal groups, i.e.,  $\G^*_0(\bmu)$. When it is clear from the context, we use $\G^*_{\epsilon}$ and $\G^*$ to denote $\G^*_{\epsilon}(\bmu)$ and $\G^*_0(\bmu)$, respectively. 
\end{definition}
In the GPSI problem, the learner seeks to identify the set of Pareto optimal groups from the smallest number of samples possible with high probability. More precisely, given a confidence parameter $\delta \in (0,1)$ and an $\epsilon > 0$, the algorithm $(P_t, \tau, \hat{i}_{\tau})$ should be $(\epsilon, \delta)$-PAC for any environment $\V$ with arm means tensor $\bmu$, i.e., 
\begin{align*}
    \forall \bmu: \pr_{\bmu} \left( \tau < + \infty, (\G^*_0(\bmu) \subseteq \hat{i}_{\tau} \subseteq \G^*_{\epsilon}(\bmu))^c \right) \leq \delta,
\end{align*}
while minimizing the total number of samples $\tau$. 
Let us introduce a few important quantities to characterize the notion of optimality/sub-optimality of a group. For any two vectors $\mathbf{v}, \mathbf{u} \in \mathbb{R}^D$ and $\alpha \geq 0$, define
\begin{align*}
    m(\mathbf{v}, \mathbf{u}) \triangleq \left[\min_{d \in [D]} \mathbf{u}^d - \mathbf{v}^d \right]^+ \quad \text{and} \quad M_{\alpha}(\mathbf{v}, \mathbf{u}) \triangleq \left[\max_{d \in [D]} \mathbf{v}^d - \mathbf{u}^d + \alpha\right]^+,
\end{align*}
where $\forall x \in \mathbb{R} : [x]^+ \triangleq \max(x,0)$. We denote $M_0(\mathbf{v}, \mathbf{u})$ by $M(\mathbf{v}, \mathbf{u})$. 

Note that $m(\mathbf{v}, \mathbf{u})$ is the minimum amount that needs to be added component-wise to $\mathbf{v}$ such that $\mathbf{u}$ does not strictly dominate $\mathbf{v}$. Clearly, $m(\mathbf{v}, \mathbf{u})=0$ if $\mathbf{v} \nprec \mathbf{u}$. Similarly, $M_{\alpha}(\mathbf{v}, \mathbf{u})$ is the minimum amount that needs to be added component-wise to $\mathbf{u}$ such that $\mathbf{u}$ weakly dominates $\mathbf{v}+\alpha$, and if the domination already holds, $M_{\alpha}(\mathbf{v}, \mathbf{u})=0$.

For a group $\G_i$ with $i \notin \G^*$, the sub-optimality gap $\Delta_i$ is defined as
\begin{align*}
    \Delta_i \triangleq \max_{j \in \G^*} m(\br_i, \br_j).
\end{align*}
This gap denotes the minimum amount that must be added to the efficiency vector of the group $\G_i$ to make it Pareto optimal.
For a group $\G_i$ with $i \in \G^*$, the gap $\Delta_i$ is defined as
\begin{align*}
    \Delta_i \triangleq \min(\Delta_i^+, \Delta_i^-),
\end{align*}
where 
\begin{align*}
    \Delta_i^+ & \triangleq \min_{j \in \G^* \backslash \{i\}} \min (M(\br_i, \br_j), M(\br_j, \br_i)), \\
    \Delta_i^- & \triangleq \min_{j\notin \G^*} M(\br_j, \br_i) + 2\Delta_j.
\end{align*}
This gap denotes the minimum estimation error in entries of $\br_i$ that leads to misidentification of $\G^*$. 
Additionally, define an in-group arm gap $\Delta_{i,j}$ for each arm $(i,j)$ as
\begin{align} \label{eq: in-group gap}
    \Delta_{i,j} \triangleq m(\bmu_{i,j}, \br_i).
\end{align}
This gap denotes the sub-optimality of arm $(i,j) $'s reward vector compared to its group's efficiency vector. 

\subsection{Failure of Naive Approach} \label{sec: failure} 
    Since the comparison among groups is based solely on their efficiency vectors, a naive approach would suggest exploring each group adequately and estimating its efficiency vector with high accuracy, then comparing these estimates to determine the set of optimal groups. However, this approach is fundamentally flawed. For groups with large gaps $\Delta_i$, it is possible to determine whether they are Pareto optimal with a small number of samples without requiring high-accuracy estimates of their efficiency vectors. Consequently, an optimal algorithm should avoid excessive pulls from arms in such groups. Thus, a naive strategy that spends equal effort on all groups would clearly have suboptimal sample complexity.
    
\subsection{Triple Elimination Algorithm}

We now introduce our algorithm for the GPSI problem, called \textit{Triple Elimination (TE)}. This algorithm belongs to the class of elimination-based algorithms, which have been widely studied and applied in various identification problems in both single-objective and multi-objective bandit settings \cite{elimination-even2006action, psi-auer2016pareto, SH-karnin2013almost, garivier2016maximin}. These algorithms operate by maintaining an active set of arms and progressively eliminating those whose optimality can be determined with high probability.  
Our algorithm generalizes this approach by simultaneously maintaining
(i) an active set of groups,  
(ii) an active set of dimensions for each active group, and  
(iii) an active set of arms for each active dimension within each active group.
Pseudo-code of the TE algorithm is given in Algorithm \ref{algo: te}.
The algorithm proceeds in rounds. At round $r$, it collects one sample from each active arm and executes three elimination phases as follows.

\paragraph{Group Elimination.}  
The group elimination phase follows an accept/reject mechanism, meaning that at each round, the algorithm eliminates groups that have demonstrated poor performance and are unlikely to belong to $\G^*$. 
Additionally, the algorithm stops sampling arms from groups $\G_i$ that are likely to be optimal with high probability when higher estimation accuracy of their efficiency vectors is unnecessary for determining the optimality of another group. More precisely, if there exists a group $\G_j$ that appears non-optimal but has an efficiency vector close to that of $\G_i$, additional sampling from the arms in $\G_i$ is required to determine the optimality of $\G_j$. If no such groups exist, $\G_i$ can be safely removed from the active set of groups and added to the optimal set.

\paragraph{Dimension Elimination.}  
For each active group $\G_i$, the algorithm maintains the active dimensions set $\D_i$. This set represents the dimensions from which additional samples are needed to determine the optimality of $\G_i$. Before presenting a formal definition that characterizes these dimensions, let us denote by $\bmuh(r)$, the empirical arm means tensor formed from the samples collected up to round $r$, and by $\rh(r)$, the empirical efficiency vectors of the groups, computed based on $\bmuh(r)$.

\begin{definition}[Dimension Resolution] \label{def: resolution}  
At round $r$ of the algorithm, given the set of active groups $G$, for each $i \in G$, we say dimension $d$  of $\G_i$ is \textit{resolved}  if  
\begin{equation*}  
    \forall j \in G\setminus\{i\}: \left| \hat{R}^d_j(r) - \hat{R}^d_i(r)\right| \geq 4 \beta(r, \delta) + \epsilon,
\end{equation*}  
where for confidence parameter $\delta$ and round $r$,  $\beta(r, \delta)$ is defined as  
\begin{align} \label{eq: beta}
    \beta(r, \delta) \triangleq \sqrt{\frac{2 \log (4 NKD r^2/\delta)}{r}}. 
\end{align}  
\end{definition}  
According to this definition, if dimension $d$ is resolved for group $\G_i$ given an active set of groups, the algorithm can infer the sign of $R^d_j - R^d_i$ for each $j$ with high probability, eliminating the need for further sampling. At each round $r$, TE eliminates the resolved dimensions of each active group from its active dimensions set. Moreover, if a dimension $d$ is eliminated for a group $\G_i$ at round $r$, the algorithm \textit{adheres to} this value of $\rhdi{r}$ and does not further update it based on the samples collected for the other dimensions in subsequent rounds.

\paragraph{Arm Elimination.}  
In this phase, the algorithm removes arms that no longer appear to be influential. For each active group and each of its active dimensions, these are the arms with a large estimated gap between their mean reward and the efficiency vector of the group in that dimension. Note that the algorithm continues sampling an arm as long as it remains active in at least one active dimension.  

\begin{algorithm}[h]
    \caption{Triple Elimination (TE) Algorithm}\label{algo: te}
    \begin{algorithmic}[1]      
        \STATE \textbf{Input.} Set of arms $\A$ and the parameters $\delta$, $\epsilon$, $\beta(r, \delta)$.
        \STATE \textbf{Initialization.} Initialize $G=[N]$ and for each $i \in [N]$, set $\D_i = [D]$, and for each $i \in [N], d \in [D]$, set $\A_{i,d} = \{(i,1), (i,2), \ldots, (i,K)\}$, which denote the active groups, active dimensions for each group, and active arms for each dimension within each group, respectively. Initialize $r=1$ and $P = \emptyset$. 
        \WHILE{$G \neq \emptyset$}
            \STATE For each $i \in G$, pull each arm $(i,j)$ such that $\exists d: (i,j) \in \A_{i,d}$ once and update empirical arm means tensor $\bmuh(r)$ and efficiency vectors $\rh_i(r)$.
            \FOR {$i \in G $}
                \STATE Remove $i$ from $G$ if $\exists j \in G: m\left(\rh_i(r), \rh_j(r)\right) \geq 2\beta(r, \delta)$.
            \ENDFOR
            \FOR  {$i \in G \text{ and } d \in \D_i$}
                \STATE Remove dimension $d$ from $\D_i$ if dimension $d$ of $\G_i$ is resolved based on definition \ref{def: resolution}, and adhere to the value of $\rhdi{r}$.
            \ENDFOR
            \FOR {$i \in G, d \in D_i, \text{ and } (i,j) \in \A_{i,d}$}
                \STATE Remove $(i,j)$ from $\A_{i,d}$ if $ \muh^d_{i,j}(r) \leq \hat{R}^d_i(r) - 2\beta(r, \delta)$.
            \ENDFOR
            \STATE $P_1 \leftarrow \{i \in G \big| \forall j \in G \setminus \{i\}: M_{\epsilon}\left(\rh_i(r), \rh_j(r) \right) \geq 2 \beta(r, \delta) \}$.
            \STATE $P_2 \leftarrow \{ i \in P_1 \big| \forall j \in G \setminus P_1:  M_{\epsilon}\left(\rh_j(r), \rh_i(r) \right) \geq 2 \beta(r, \delta) \}$.
            \STATE $P \leftarrow P \cup P_2$.
            \STATE $G \leftarrow G \setminus P_2$.
            \STATE $r \leftarrow r+1$. 
        \ENDWHILE
        \STATE \textbf{Output.} Set $P$. 
    \end{algorithmic}
\end{algorithm}
The algorithm terminates when no active groups remain.
The following theorem proves the correctness of Algorithm \ref{algo: te} and provides an upper bound on the total number of arm pulls. 

\begin{restatable}[Correctness and Upper Bound for Algorithm \ref{algo: te}]{theorem}{teUpperBound} \label{thm : te upper}
    For any environment $\V$ with arm means tensor $\bmu$, any $\epsilon > 0$ and $\delta \in (0,1)$, Algorithm \ref{algo: te} with parameter $\beta(r, \delta)$ defined in \eqref{eq: beta} is $(\epsilon, \delta)$-PAC. Moreover, with probability $1 - \delta$ the total number of arm pulls is at most
    \begin{equation*}
        \sumij \frac{C}{\tilde{\Delta}_{i,j}^2} \log\left( \frac{NKD}{\delta \tilde{\Delta}_{i,j}} \right),
    \end{equation*}
    where $C$ is a universal constant and $\tilde{\Delta}_{i,j} \triangleq \max(\Delta_{i,j}, \Delta_i, \epsilon)$. 
\end{restatable}

\begin{remark}
    In the proof of Theorem \ref{thm : te upper} in Appendix \ref{sec: gpsi}, we demonstrate that the same sample complexity upper bound holds even without performing the dimension elimination phase. This suggests that in most instances where dimension elimination occurs, the algorithm can achieve better performance than this theoretical upper bound.
\end{remark}

\subsection{Lower Bound}

Next, we derive a lower bound on the sample complexity of any $(\epsilon, \delta)$-PAC algorithm that solves GPSI, and show that for some instances, TE is optimal up to constant and logarithmic factors. 

\begin{restatable}[GPSI Lower Bound]{theorem}{gpsiLower} \label{thm : gpsi lower} 
            For any $\delta, K , N$ and $D$ with $D,N > 1$, there exists a non-trivial class of environments $\V$ with $1$-Gaussian reward distributions such that for small $\epsilon > 0$, any $(\epsilon, \delta)$-PAC algorithm with almost-surely finite stopping time that solves GPSI on $\V$ requires 
            \begin{align*}
                \Omega \left( \sumij \frac{1}{\tilde{\Delta}_{i,j}^2} \log\left( \frac{1}{2.4\delta} \right) \right)
            \end{align*}
            expected number of samples, where $\tilde{\Delta}_{i,j}$ is defined in Theorem \ref{thm : te upper}.
\end{restatable}

\section{Linear Best Group Identification Problem} \label{sec: lbgi}

This section addresses the \textit{ linear best group identification (LBGI)} problem. We begin by formally defining the problem and subsequently present our proposed algorithm, equal effect confidence bound, along with its sample complexity upper bound. Lastly, we establish an instance-dependent lower bound on the sample complexity of any algorithm that solves this problem. 

As discussed in the introduction, the LBGI problem arises when different dimensions of the reward are weighted differently, with the weights known to the learner. Let $\mathbf{w} \in \mathbb{(R^+)}^D$ represent the weight vector, where the weight associated with dimension $d$ of the reward is $w^d$. Hence, the expected weighted reward of a group $\G_i$ with efficiency vector $\br_i$ is given by $\br_i^{\top} \bw$.

We assume the tensor $\bmu$ is such that the group with the highest expected reward, denoted by $\istarmu$, is unique. 
Consistent with the definitions in the previous section, an algorithm $(P_t, \tau, \hat{i}_{\tau})$ is said to be $(\epsilon, \delta)$-PAC for the LBGI problem if, for any $\bmu$ that implies a unique best group, it satisfies:
\begin{align*}
    \pr_{\bmu} \left(\tau < + \infty, (\br_{\hat{i}_{\tau}} + \epsilon)^{\top} \bw < \br_{\istarmu} ^{\top} \bw \right) \leq \delta.
\end{align*}
In this paper, we focus on the case with $\epsilon = 0$, and say an algorithm is $\delta$-correct if it is $(0, \delta)$-PAC.

For a non-optimal group $\G_i$, the sub-optimality gap $\Delta_i$ is defined as
\begin{align*}
    \Delta_i \triangleq \inf\{\alpha > 0 \mid (\br_i + \alpha)^{\top}\bw > \br_{\istarmu}^{\top} \bw \} = \frac{\br_{\istarmu}^{\top} \bw - \br_i^{\top} \bw}{\wlone},
\end{align*}
which shows the minimum amount that needs to be added to entries of $\br_i$ to make $\G_i$ the best group.
The optimal group $\G_{\istarmu}$, has the smallest gap among all other groups, i.e., $\Delta_{\istarmu} \triangleq \min_{i \neq \istarmu} \Delta_i$.
For any arm $(i,j)$ in group $\G_i$ and for any real number $\alpha > 0$, we use the notation $\br(i,j,\alpha)$ to denote the efficiency vector of group $\G_i$ when $\alpha$ is added to all dimensions of the reward means of arm $(i,j)$ ($\bmu_{i,j} \leftarrow \bmu_{i,j} + \alpha)$. Using this notation, for each arm $(i,j)$ we define
\begin{align*}
    \alpha_{i,j} \triangleq \inf \{ \alpha > 0 \mid \br(i,j,\alpha)^{\top} \bw >  \br_{\istarmu}^{\top} \bw \}.
\end{align*}
This value is equal to the minimum value that must be added to the reward means of arm $(i,j)$ to make the reward of its group (i.e., $\G_i$) higher than the reward of $\G_{\istarmu}$. Clearly for arms in $\G_{\istarmu}$, $\alpha_{i,j} = m(\bmu_{\istarmu, j}, \br_{\istarmu})$. Now for each arm $(i,j)$, we define the gap $\Delta_{i,j} = \max(\alpha_{i,j}, \Delta_i)$. 
Note that for the arms with $i \neq \istarmu$, since $\br_i \geq \bmu_{i,j}$, $\alpha_{i,j} \geq \Delta_i$ which implies $\Delta_{i,j} = \alpha_{i,j}$. 

The naive approach, discussed in Section \ref{sec: failure}, also fails for the LBGI problem for the same reason. That is, groups with large gaps $\Delta_i$ can be determined as suboptimal with relatively few samples, and an optimal algorithm should avoid unnecessarily drawing samples from such groups.

\subsection{Equal Effect Confidence Bound Algorithm}

Now we describe our algorithm called \textit{equal effect confidence bound (EECB)} for the LBGI problem. 
Similar to Triple Elimination (TE), EECB is an elimination-based algorithm. 
The algorithm operates in rounds, maintaining a set of active groups $G$ and a set of active arms $\A_d$ for each dimension $d$ (this includes arms from all groups). The set $G$ consists of groups whose non-optimality has not yet been determined, and for each dimension $d$, the set $\A_d$ contains arms whose estimated mean reward is close to the estimated efficiency vector of their group in dimension $d$.
Similar to the previous section, let $\bmuh(r)$ and $\rh_i(r)$ denote the empirical arm means tensor and empirical efficiency vector of $\G_i$ computed based on the samples collected until round $r$. Moreover, let $N_{i,j}(r)$ denote the number of samples seen from arm $(i,j)$ until that round. 

At each round $r$, the algorithm focuses on a dimension $\dr$ which maximizes the value $w^d \betnd{r}$, where $\bet{r}$ is defined in \eqref{eq: beta} and the vector $\bn(r) = [n^1(r), \ldots , n^D(r)]$ represents the number of rounds each dimension has been focused on so far. Then, the algorithm pulls each arm $(i,j) \in \A_{\dr}$ that satisfies $N_{i,j}(r) = n^{\dr}(r)$ once and updates $\bmuh$ and $\rh$.  

By selecting dimensions in this manner,
$n^{\dr}(r+1) = n^{\dr}(r)+1$ which implies that $\beta(n^{\dr}(r+1), \delta) < \beta(n^{\dr}(r), \delta)$ since $\beta(r, \delta)$ is decreasing in $r$. Thus,  at each round $r$, the algorithm decreases the value of $w^{\dr} \beta(n^{\dr}(r), \delta)$, which is the maximum across the dimensions.
This results in nearly equal values $w^d \betnd{r}$ for different dimensions $d$ in the long run.

The algorithm performs two elimination phases in each round. In the first phase, it eliminates the groups with a large estimated gap, and in the second phase, for each dimension $d$, the algorithm eliminates the arms with a large gap in the $d$-th dimension of their estimated reward means compared to the best arm in the group. Algorithm \ref{algo: eecb} contains the pseudo-code of EECB. The algorithm halts when only one active group remains.
The following theorem establishes the correctness of Algorithm \ref{algo: eecb} and provides an upper bound on the total number of arm pulls.

\begin{algorithm}[t]
    \caption{Equal Effect Confidence Bound (EECB) Algorithm}\label{algo: eecb}
    \begin{algorithmic}[1]      
        \STATE \textbf{Input.} Set of arms $\A$ and the parameters $\delta, \beta(r, \delta)$.
        \STATE \textbf{Initialization.} Set $G = [N]$ and for each $d \in [D]$, $\A_d = \A $ to show the active groups and the active arms in dimension $d$. Pull each arm once, build $\hat{\bmu}(1)$ and $\hat{\br}(1)$. Set $r=1$ and $\mathbf{n}(1) = (1, 1, \ldots, 1)$.
        \WHILE{$ \lvert G \rvert > 1$}
            \FOR {$i \in G $}
                \STATE Remove $i$ from $G$ if $\exists j \in G: \rh_j(r)^{\top} \bw - \rh_i(r)^{\top} \bw > 2 \sum_{d \in [D]} w^d \betnd{r}$.
            \ENDFOR
            \FOR {$d \in [D]$}
                \FOR {$(i,j) \in \A_d$}
                    \STATE Remove $(i,j)$ from $\A_d$ if $\exists (i,l) \in \A_d: \muh_{i,l}^d(r) > \muh_{i,j}^d(r) + 2 \betnd{r} $.
                \ENDFOR
            \ENDFOR
            \STATE $\dr = \arg \max_{d \in [D]} w^d \betnd{r}$. (In the case of equal values, choose one arbitrarily) 
            \FOR {$i \in G $}
                \STATE Pull each arm $(i,j) \in \A_{\dr}$ with $N_{i,j}(r) = n^{\dr}(r)$ once and update $\bmuh$ and $\hat{\br}$ .
            \ENDFOR
            \STATE $\mathbf{n}(r+1) \leftarrow \mathbf{n}(r) + \mathbf{e}_{\dr}$, where $\mathbf{e}_{\dr}$ is the unit vector in $\mathbb{R}^D$ with $1$ in $\dr$-th element.
            \STATE $r \leftarrow r+1$. 
        \ENDWHILE
        \STATE \textbf{Output.} Set $G$.
    \end{algorithmic}
\end{algorithm}

\begin{restatable}[Correctness and Upper Bound for Algorithm \ref{algo: eecb}]{theorem}{eecbUpperBound} \label{thm : eecb upper}
    For any environment $\V$ with arm means tensor $\bmu$ and $\delta \in (0,1)$, Algorithm \ref{algo: eecb} with parameter $\beta(r, \delta)$ defined in \eqref{eq: beta} is $\delta$-correct. Moreover, with probability $1 - \delta$ the total number of arm pulls is as most
    \begin{equation} \label{eq: eecb-upper}
        \sumij \frac{C'}{(\Delta_{i,j}/D)^2} \log\left( \frac{NKD}{\delta \Delta_{i,j}} \right),
    \end{equation}
    where $C'$ is a universal constant. 
\end{restatable}

\subsection{Lower Bound}

In the following theorem, we derive a lower bound on the sample complexity of any $\delta$-correct algorithm that solves LBGI.

\begin{restatable}[LBGI Lower Bound]{theorem}{lbgiLower} \label{thm : lbgi lower} 
            For any $\delta \in (0,1)$ and tensor $\bmu$, there exists an environment $\V$ with arm means tensor $\bmu$ and $1$-Gaussian reward distributions, such that any $\delta$-correct algorithm with almost-surely finite stopping time that solves the LBGI problem, requires 
            \begin{align*}
                \Omega \left( \frac{1}{\Delta_{i,j}^2} \log\left( \frac{1}{2.4 \delta} \right) \right)
            \end{align*}
            expected number of samples from each arm $(i,j)$ with $i \neq i^*(\bmu)$. Moreover, the algorithm requires 
            \begin{align*}
                \Omega \left( \frac{1}{\Delta_{\istarmu}^2} \log\left( \frac{1}{2.4 \delta} \right) \right)
            \end{align*}
            expected number of samples in total from arms in group $\G_{\istarmu}$.
\end{restatable}

\begin{remark} \label{rem: LBGI-upper}  
    To compare the lower and upper bound, in Appendix \ref{apx: eecb-upper}, we show that the term $\frac{\Delta_{i,j}}{D}$ in the denominator of \eqref{eq: eecb-upper} can be replaced with
    \begin{align} \label{eq: lbgi-larger-gaps}  
        \min_{d \in [D]}{\frac{\Delta_{i} \wlone}{D w^d} \; + (R^d_i - \mu^d_{i,j})},  
    \end{align}  
    which is significantly larger in some cases. To compare these two quantities, note that even after dividing the second term in \eqref{eq: lbgi-larger-gaps} by $D$, i.e.,  
    $$  
    \min_{d \in [D]} \frac{1}{D} \left( \frac{\Delta_i \wlone}{w^d} + (R^d_i - \mu^d_{i,j}) \right),  
    $$  
    the result remains greater than or equal to $\frac{\Delta_{i,j}}{D}$. This difference arises because, by definition, $\Delta_{i,j}$ represents the minimum value that must be added uniformly to \emph{all} dimensions of $\bmu_{i,j}$ to make $\G_i$ the best group. In contrast, $\frac{\Delta_i \wlone}{w^d} + (R^d_i - \mu^d_{i,j})$ represents the minimum value required to be added to $d$-th dimension of $\bmu_{i,j}$ to achieve the same effect. Clearly, the latter is larger and can be up to $D$ times larger in some cases. This result indicates that, in certain instances, our algorithm is near-optimal (up to constant and logarithmic factors) for all non-optimal groups.  
\end{remark}

\section{Experiments} \label{sec: exp}

In this section, we present some numerical evaluations of our two algorithms on various instances. For all experiments, we set $\delta = 0.01$, and the reward distribution of each arm is a random vector where each dimension follows an independent standard Gaussian distribution. For each combination of $N, K,$ and $D$ in both GPSI and LBGI problems, we generate the entries of the reward means vector $\bmu$ randomly from the interval $[0,1]$, ensuring that  $\Delta_{\min} = \min_{i \in [N]}$ is not significantly small to avoid extremely difficult instances.
The results are averaged over $20$ runs, with the standard deviation of the stopping time displayed in the plots. Note that in all experimental setups, we used the exact value of $\bet{r}$ as defined in \eqref{eq: beta}. While this choice ensures the theoretical guarantees of correctness, it tends to be overly conservative in practice. Empirically, using smaller values than the prescribed $\bet{r}$ results in error rates still below $\delta$, but with fewer samples.


\subsection{GPSI Experiments}  

For GPSI experiments, in all instances, we set $\epsilon = 0.01$ and $\Delta_{\min} > 3 \epsilon$. Since no existing algorithm in the literature addresses this problem, we compare our triple elimination (TE) algorithm with two simplified versions of it and a naive baseline method:

\textbf{Arm-Group Elimination (AGE).}  
This algorithm follows a similar approach to TE but omits the dimension elimination phase, performing only group and arm eliminations. Comparing AGE with TE highlights the practical impact of dimension elimination.

\textbf{Group Elimination (GE).}  
GE further simplifies AGE by discarding the arm elimination phase as well and performing elimination solely at the group level. 

\textbf{Uniform Sampling (UniS).}  
This naive algorithm pulls each arm uniformly $\frac{8}{\epsilon^2} \log\left( \frac{2NKD}{\delta} \right)$ times and, based on the estimated arm means tensor, outputs the set of all optimal arms along with non-optimal arms whose gap is less than $\epsilon$.

In Appendix \ref{apx: exp-baselines-correctness}, we provide a theoretical justification demonstrating that all three baseline algorithms are $(\epsilon, \delta)$-PAC. We compare their performance under two different scenarios.  

\textbf{Varying $N$.}  
In this experiment, we evaluate the average sample complexity of the algorithms on instances with $K = 6$, $D = 3$, and varying values of $N$. In all instances, the number of Pareto optimal groups is set to $\lceil 0.3N \rceil$. Figure \ref{fig: vary-n} demonstrates the performance gap between TE and other baselines.

\textbf{Varying $K$.}  
This experiment analyzes the average sample complexity of the algorithms on instances with $N = 5$, $D = 3$, and varying values of $K$. The number of Pareto optimal groups is fixed at $2$ across all instances. Figure \ref{fig: vary-k} illustrates the average sample complexity of the different algorithms, showing that TE consistently achieves the best performance. 


\begin{figure}[t!]
        \centering
        \begin{subfigure}[b]{0.45\textwidth}
            \centering
            \includegraphics[width = \textwidth]{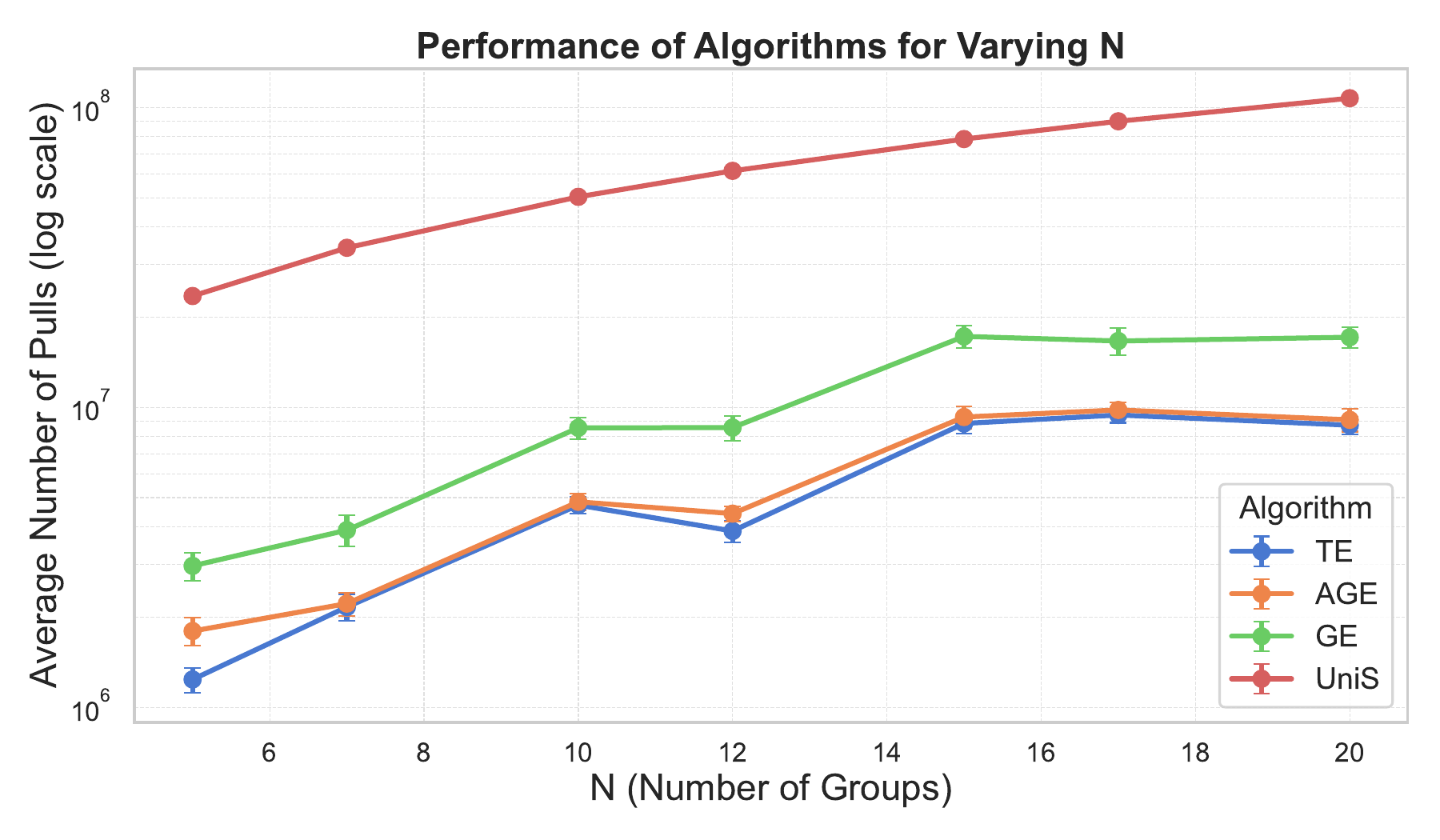}
            \caption{Average stopping time of different algorithms on instances with varying $N$.}
            \label{fig: vary-n}
        \end{subfigure}
        \begin{subfigure}[b]{0.45\textwidth}
            \centering
            \includegraphics[width = \textwidth]{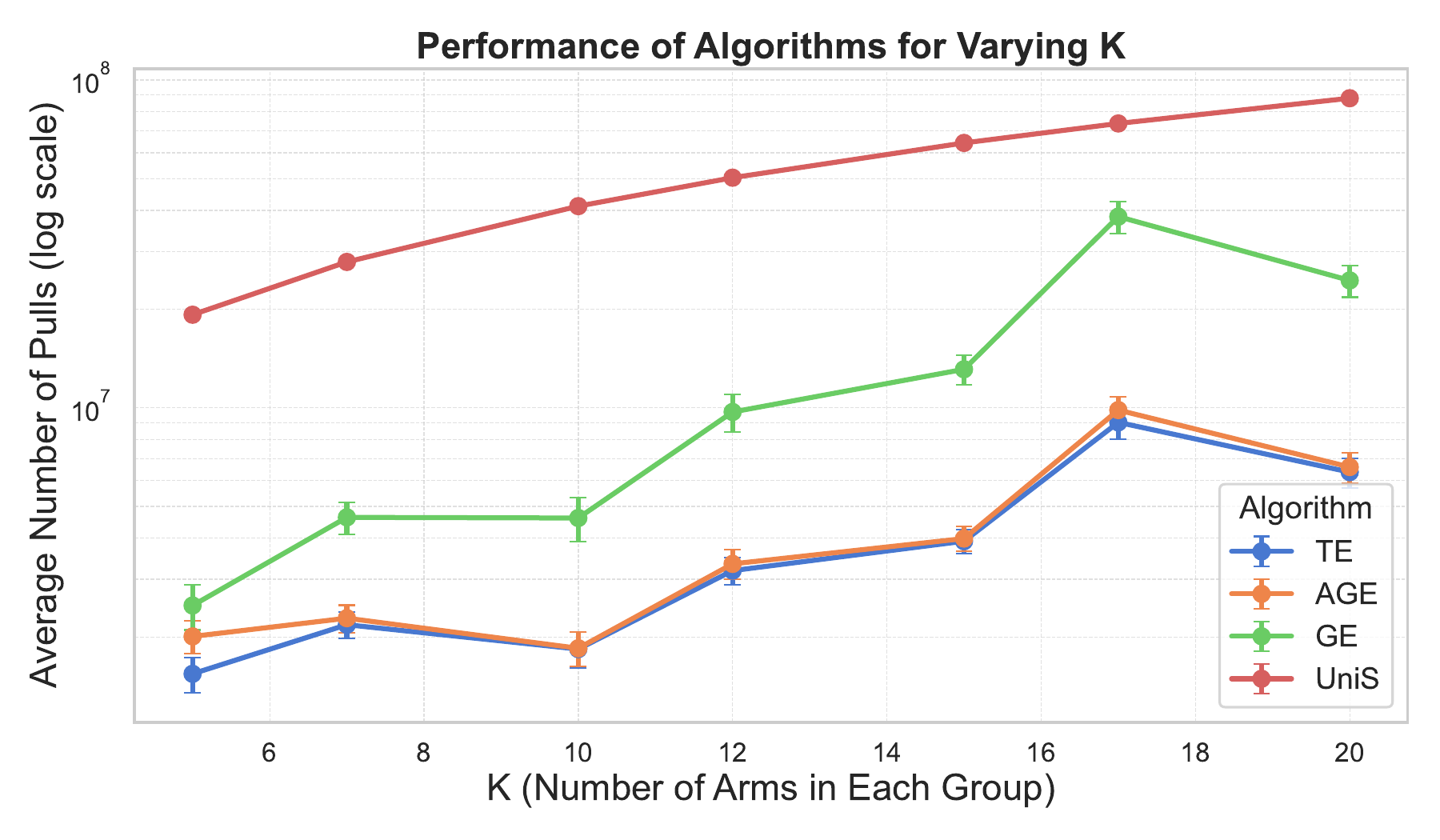}
            \caption{Average stopping time of different algorithms on instances with varying $K$.}
            \label{fig: vary-k}
        \end{subfigure}
        \caption{Results of the experiments for the GPSI problem.}
        \vspace{-0.5cm}
        \label{fig: gpsi}
\end{figure}

\subsection{LBGI Experiments}

For LBGI experiments, there is no existing state-of-the-art algorithm for direct comparison. Therefore, we compare our proposed algorithm (EECB) with one baseline algorithm.  
\paragraph{Triple Elimination for LBGI (TEL).}  
Since the weight vector $\bw$ is positive, the efficiency vector of the optimal group, $\br_{\istarmu}$, is necessarily among the set of Pareto optimal efficiency vectors. Inspired by this observation, the TEL algorithm initially ignores $\bw$ and applies the triple elimination algorithm with $\epsilon = 0.01$ to identify the set of Pareto optimal groups $\G^*$ with high probability. Then, using the estimated efficiency vectors of these groups at the end of the process, it selects the group $\G_i$ that maximizes $\hat{\br}_i^{\top} \bw$.  
Unlike EECB, TEL has no theoretical guarantee for $\delta$-correctness.

Table \ref{tab: lbgi} compares the average stopping time of EECB and TEL on a random instance with $N=5$, $K=5$, and $D=3$, using five different weight vectors:  
$
\bw_1 = (0.1, 0.1, 1), \quad \bw_2 = (1, 1, 0.1), \quad \bw_3 = (0.1, 1, 1), \quad \bw_4 = (1,1,1), \quad \text{and} \quad \bw_5 = (1,2,3).
$  
The results indicate that EECB significantly outperforms TEL across all weight vectors.
\begin{table}[ht]
\centering
\begin{tabular}{|c|c|c|c|c|c|}
\hline
\textbf{Algorithm} & \textbf{$\bw_1$} & \textbf{$\bw_2$} & \textbf{$\bw_3$} & \textbf{$\bw_4$} & \textbf{$\bw_5$} \\
\hline
\textbf{EECB} & 225766 & 534801 & 209423 & 124603 & 141455 \\
\textbf{TEL}  & 7176441 & 7176441 & 7176441 & 7176441 & 7176441 \\
\hline
\end{tabular}
\vspace{0.5cm}
\caption{Comparison of the average stopping time of EECB and TEL algorithms on a single instance with varying weight vectors $\bw$. Since TEL ignores $\bw$ during sampling, the average stopping time remains constant across different weight vectors.}
\label{tab: lbgi}
\end{table}

\vspace{-0.5cm}
\section{Conclusion} \label{sec: conclusion}

In this work, we introduced the best group identification problem in multi-objective MAB setting and proposed efficient algorithms for two key formulations in the fixed confidence setting: group Pareto set identification and linear best group identification. We established sample complexity upper and lower bounds for both settings. Our experiments showcase the strong performance of our proposed algorithms. These results advance pure exploration in multi-objective MAB settings and open new directions for future research.

An interesting direction for future work is to extend the LBGI problem to scenarios where the weight vector $\bw$ is not fully known, but only partially specified. In real-world applications, the relative importance of different reward dimensions is often uncertain or subject to estimation, making it valuable to develop algorithms that can adapt to incomplete or uncertain weight information.


\bibliographystyle{alpha}
\bibliography{bibliography}

\newpage
\onecolumn

\title{Appendix}
\maketitle

\appendix

The appendix is organized as follows. In Appendix \ref{apx: related-work}, we discuss the related work in more detail. Appendix \ref{apx: gpsi} contains the proofs related to the GPSI problem. Appendix \ref{apx: lbgi-proofs} presents the proofs for the LBGI problem. Appendix \ref{apx: exp-baselines-correctness} provides a discussion on the correctness of the baseline algorithms used in the experiments.

\section{Discussion on Related Work} \label{apx: related-work}

\paragraph{Pure Exploration in Multi-Objective Bandits.}
Multi-objective online learning problems have been studied in different settings such as Pareto Set Identification (PSI) in black-box optimization \cite{blackbox1-deb2002fast, blackbox2-knowles2006parego, blackbox3-zuluaga2013active}, reinforcement learning \cite{rl1-van2014multi, rl2-roijers2013survey, rl3-russo2024multi}, and multi-armed bandits \cite{regret1-firsrt-drugan2013designing, psi-auer2016pareto, regret-adversary-xu2023pareto}.

In the MAB setting, \cite{regret1-firsrt-drugan2013designing} was the first to consider the multi-objective environment, where the goal is to minimize regret by fairly playing the Pareto optimal arms. Regret minimization in multi-objective bandits has also been studied in the generalized linear setting \cite{glinear-lu2019multi} and the adversarial setting \cite{regret-adversary-xu2023pareto}.

In the domain of pure exploration problems, \cite{MOBAI-chen2025optimal} recently studied the multi-objective best arm identification, where the goal is to identify the best arm in each individual dimension. They proposed an asymptotically optimal algorithm under the assumption that reward dimensions are independent. 
PSI is the most extensively studied pure exploration problem in multi-objective bandits. The first PSI algorithm, based on uniform sampling and an accept/reject mechanism, was introduced in \cite{psi-auer2016pareto}. \cite{psi-lucb-kone2024bandit} proposed an LUCB-based algorithm for PSI and several relaxed variants of the problem. These works primarily address the fixed-confidence setting. The fixed-budget variant of the PSI problem was explored in \cite{psi-budget-kone2024bandit}.

The PSI problem in multi-objective linear bandits, under both fixed-confidence and fixed-budget settings, was studied in \cite{psi-linear-konebandit}. 
Although the PSI problem without any condition is not yet optimally solved, asymptotically optimal algorithms have been proposed for special settings recently. In \citep{psi-garivier2024sequential}, the authors proposed an asymptotically optimal PSI algorithm for Gaussian rewards with an identity covariance matrix based on the Track-and-Stop algorithm in the best arm identification literature \cite{track-stop-garivier2016optimal}. Furthermore,\cite{psi-posterior-kone2024pareto} addresses both the unstructured PSI and the transductive linear setting with Gaussian rewards and a known covariance matrix and proposes an asymptotically optimal algorithm using posterior sampling.

\paragraph{Single-Objective Case.} 
In the single-objective case, our problem reduces to a multi-armed bandit problem where the arms are partitioned into groups, and the objective is to identify the best group. Several studies have explored this problem and closely related ones, particularly focusing on identifying a good subset of arms for a fixed error probability  \cite{garivier2016maximin, maxmin-grouped-wang2022max, subset-selection1-kalyanakrishnan2012pac, subset-selection2-kaufmann2013information, quantile-lau2023max}. 

\citep{garivier2016maximin} modeled an action identification problem in game theory as a MAB problem with grouped arms, where the objective is to find the group with the highest minimum arm mean. Similarly, \cite{maxmin-grouped-wang2022max} studies the same problem but allows for overlapping groups. In \cite{quantile-lau2023max}, the focus is on a scenario where each group contains infinitely many arms, and the goal is to identify the group with the highest $(1 - \alpha)$-quantile.

These subset selection problems in single-objective bandits can all be viewed as special cases of a general identification problem. In this problem, the space of arm mean vectors is partitioned into non-overlapping regions (the case of overlapping regions is studied in \cite{mutiple-correct-answers-degenne2019pure} and requires more involved analysis). The objective is to minimize the number of samples required to identify the region the actual vector of means belongs to. This generalized problem has been studied under various names and assumptions. In \cite{kaufmann2021mixture, frank-wolf-wang2021fast, partition-id-juneja2019sample}, the authors proposed algorithms with asymptotically optimal sample complexity. Additionally, \cite{general-samp-chen2017nearly} provides a non-asymptotic analysis and corresponding bounds for this problem. Studying this general problem in the multi-objective setting can be seen as an important future direction of this work.

\section{Proofs of Section \ref{sec: gpsi}} \label{apx: gpsi}

First, we introduce some notation. Let $G_1(r)$ denote the set of active groups at round $r$ of Algorithm \ref{algo: te} after eliminating the suboptimal groups (lines $6$ to $8$ of the pseudo-code). We refer to this elimination phase as the group rejection phase.  
Additionally, for each group $\G_i$ and dimension $d$, let $(i,i^*_d)$ represent the arm with the highest mean value in dimension $d$ among all arms in group $\G_i$ (if there are several arms with maximum value of dimension $d$, we define $(i, i^*_d)$ to be the one with the smallest ingroup index). More precisely,  
\begin{align} \label{def: istard}  
    i^*_d \triangleq \argmax_{j \in [K]} \mu^d_{i,j}.  
\end{align}

\subsection{Proof of Theorem \ref{thm : te upper}}

\teUpperBound*

\begin{proof}
    To prove the theorem, first, we provide two lemmas. 

    \begin{lemma} \label{lem: sub-gaussian}
        For a $1$-sub-Gaussian random variable $X$ with $\mathbb{E}[X] = \mu$ and for any $\delta \in (0,1)$, the following holds:
        \begin{equation*}
            \pr \left(\forall r \in \mathbb{N}: \left|\muh(r) - \mu \right| < \sqrt{\frac{2 \log(4r^2/\delta)}{r}} \right) \geq 1 - \delta,
        \end{equation*}
        where $\muh(r)$ is the empirical mean estimated using $r$ samples of $X$. 
    \end{lemma}
    \begin{proof}
        By Hoeffding's inequality and a union bound on all natural numbers
        \begin{align*}
            \pr \left(\exists r \in \mathbb{N}: \left| \muh(r) - \mu \right| \geq \sqrt{\frac{2 \log(4r^2/\delta)}{r}} \right) &\leq 2 \sum_{r \in \mathbb{N}} \exp \left( -\frac{r \left(\frac{2 \log(4r^2/\delta)}{r} \right)}{2} \right) \\
            & = \sum_{r \in \mathbb{N}} 2 \left( \frac{\delta}{4r^2} \right) = \frac{\delta}{2} \sum_{r \in \mathbb{N}} \frac{1}{r^2} = \frac{\delta}{2} \times \frac{\pi^2}{6} < \delta.
        \end{align*}
    \end{proof}

    \begin{lemma} \label{lem: n-beta}
        For each $ 0 < \Delta \leq 1$ and $\delta \in (0,1)$, if $r \geq \frac{30 \log\left( \frac{NKD}{\delta \Delta} \right)}{\Delta^2}$, then $\bet{r} < \Delta$.
    \end{lemma}
    \begin{proof}
        Since $\bet{r}$ is a decreasing function in $r$ for each $\delta$, it is enough to show that for $r_0 = \frac{30 \log\left( \frac{NKD}{\delta \Delta} \right)}{\Delta^2}$, $\bet{r_0} < \Delta$. To prove this we have
        \begin{align*}
            \bet{r_0} < \Delta &\Longleftrightarrow \sqrt{\frac{\log \left( 4NKD \times 30^2 \log^2\left( NKD/\delta\Delta \right) \right)}{15 \log\left(NKD/\delta \Delta \right)}} < 1 \\ 
            & \Longleftrightarrow\log \left( 4NKD \times 30^2 \log^2\left( NKD/\delta\Delta \right) \right) < 15 \log\left(NKD/\delta \Delta \right) \\
            & \Longleftrightarrow 4NKD \times 30^2 \log^2\left( NKD/\delta\Delta \right) < \left(NKD/\delta \Delta \right)^{15} \\ &\xLongleftarrow{\Delta, \delta \leq 1} 3600 < (NKD)^{12},
        \end{align*}
        where the final inequality holds because $N\geq 2$.
    \end{proof}

    Introducing the good event $\E$ as
    \begin{equation*}
        \E = \bigcap_{i \in [N]} \bigcap_{j \in [K]} \bigcap_{d \in [D]} \bigcap_{r \in \mathbb{N}} \left\{ \left| \muh^d_{i,j}(r) - \mu^d_{i,j} \right| < \beta(r, \delta) \right\},
    \end{equation*}

    Lemma \ref{lem: sub-gaussian}, together with a union bound over all arms and dimensions, imply that $\pr(\E^c) \leq \delta$. We establish the correctness of Algorithm \ref{algo: te} and derive its sample complexity bound under $\E$.  
    For the remainder of the proof, assume $\E$ holds, which implies that most of the subsequent statements are valid with probability at least $1 - \delta$.

    \begin{lemma} \label{lem: gpsi-best-arms}
        At each round $r$, for each $i \in [N]$ and $d \in [D]$, if group $\G_i$ is active and dimension $d$ is not resolved, then arm $(i,i^*_d)$ is also active in dimension $d$, where $i^*_d$ is defined in \eqref{def: istard}.
    \end{lemma}
    \begin{proof}
       We prove this by contradiction. Assume that $i^*_d$ is not active in dimension $d$ at round $r$. This implies that there must be an earlier round $r' < r$ such that  

        \begin{align*}
            \exists j : \muh^d_{i,j}(r') \geq \muh^d_{i, i^*_d}(r') + 2 \bet{r'} > \mu^d_{i,i^*_d} + \bet{r'} \geq \mu^d_{i,j} + \bet{r'} > \muh^d_{i,j}(r'),
        \end{align*}  
        
        which is a contradiction, thereby proving the statement.

    \end{proof}

    \begin{lemma} \label{lem: gpsi-arm-elim}
        At each round $r$ for each group $\G_i$ and dimension $d \in [D]$, if $d$ is not resolved for $\G_i$ until round $r$, then 
        \begin{equation*}
            \left| \rhdi{r} - R_i^d \right| < \beta(r, \delta),
        \end{equation*}
        and if it was resolved in round $r' < r$, then 
        \begin{equation*}
            \left| \rhdi{r} - R_i^d \right| < \beta(r', \delta).
        \end{equation*}
    \end{lemma}
    \begin{proof}
        First, assume $d$ is not resolved for group $\G_i$. Then by Lemma \ref{lem: gpsi-best-arms}, arm $(i,i^*_d)$ is active. Then

        \begin{align*}
            &\hat{R}^d_{i}(r) \geq \muh^d_{i, i^*_d}(r) > R^d_i - \bet{r}  \Longrightarrow \rhdi{r} > R^d_i - \bet{r} \\
            &\forall j : \muh^d_{i, j}(r) < \mu^d_{i,j} + \bet{r} \leq R^d_i + \bet{r} \Longrightarrow \rhdi{r} < R^d_i + \bet{r},
        \end{align*}

        which proves the statement. For the case where $d$ is eliminated at round $r'$, note that by above inequality $\left| \rhdi{r'} - R_i^d \right| < \beta(r', \delta)$ and as the algorithm adheres to the value of $\rhdi{r'}$ after the resolution, this is true for all $ r \geq r'$.
    \end{proof}

    \begin{lemma} \label{lem: gpsi-dim-elim}
        For each $i, j \in [N]$ and $d \in [D]$, if $r'$ is the first round in which dimension $d$ is eliminated for either group $\G_i$ or $\G_j$, then, depending on which of $\rhdi{r'}$ and $\rhdj{r'}$ is larger, we have:
        \begin{align} \label{eq: gpsi-dim-slim}
            \text{ if } \rhdi{r'} - \rhdj{r'} \geq 4 \bet{r'} + \epsilon \Longrightarrow \forall r > r' : 
            \begin{cases}
                R^d_i - R^d_j > 2 \bet{r} + \epsilon, \\
                \rhdi{r} - \rhdj{r} > 2 \bet{r} + \epsilon,
            \end{cases} 
        \end{align}
        and the inverse is true if at round $r'$, $\rhdj{r'}$ is larger than $\rhdi{r'}$.
    \end{lemma}
    \begin{proof}
        The first inequality directly follows from Lemma \ref{lem: gpsi-arm-elim}. For the second one, if at round $r'$, dimension $d$ is eliminated for $\G_i$, then $\forall r > r': \rhdi{r} = \rhdi{r'}$. Moreover,
        \begin{align*}
            \lvert \rhdj{r} - \rhdj{r'} \rvert \leq \lvert \rhdj{r} - R^d_j \rvert + \lvert R^d_j - \rhdj{r'} \rvert < \bet{r} + \bet{r'} < 2 \bet{r'}.
        \end{align*}
        Combining these two, we obtain $\rhdi{r} - \rhdj{r} > 2 \bet{r'} + \epsilon$, then
        \begin{align*}
            \forall r > r' : \rhdi{r} - \rhdj{r} >  2 \bet{r} + \epsilon.
        \end{align*}
        If dimension $d$ is eliminated for $\G_j$, the same argument works. Similarly, for the case where at round $r'$, $\rhdj{r'}$ is larger than $\rhdi{r'}$, we have $\forall r > r' : \rhdj{r} - \rhdi{r} > 2 \bet{r} + \epsilon$, and $R^d_j - R^d_i > 2 \bet{r} + \epsilon$.
    \end{proof}

    The following lemma provides helper properties that hold during the interaction of the algorithm with the environment. 

    \begin{lemma} \label{lem: gpsi-properties}
        At each round $r$ of Algorithm \ref{algo: te}, for each $i,j \in [N]$ and $d \in [D]$, the following properties hold: 
        \begin{align}
            &\forall \alpha \in [0,\epsilon]: \rhdi{r} - \rhdj{r} \geq 2 \bet{r} - \alpha \Longrightarrow R^d_i - R^d_j > - \alpha \label{eq: prop-in-1} \\
            &\m{\rhi{r}}{\rhj{r}} \geq 2 \bet{r} \Longrightarrow \m{\br_i}{\br_j} > 0, \label{eq: prop-1} \\ 
            &\Meps{\rhi{r}}{\rhj{r}} \geq 2 \bet{r} \Longrightarrow \Meps{\br_i}{\br_j} > 0, \label{eq: prop-2} \\ 
            &\m{\rhi{r}}{\rhj{r}} \leq 2 \bet{r} \text{ and } \br_i \prec \br_j \Longrightarrow \m{\br_i}{\br_j} < 4 \bet{r} \label{eq: prop-new}\\
            &\Meps{\rhi{r}}{\rhj{r}} \leq 2 \bet{r} \Longrightarrow \M{\br_i}{\br_j} < 4 \bet{r}.  \label{eq: prop-4}
        \end{align}
    \end{lemma}
    \begin{proof}
        To prove \eqref{eq: prop-in-1}, we consider three cases: (i) Dimension $d$ is active for both $\G_i$ and $\G_j$ at round $r$. In this case Lemma \ref{lem: gpsi-arm-elim} implies 
        \begin{align*}
            R^d_i - R^d_j > \rhdi{r} - \rhdj{r}  - 2 \bet{r} \geq - \alpha.
        \end{align*}
        (ii) Dimension $d$ was eliminated for one of the groups at round $r'$ for the first time and $\rhdi{r'}$ was larger in that round. Then we have
        \begin{align*}
            \rhdi{r'} - \rhdj{r'} \geq 4 \bet{r'} + \epsilon \xLongrightarrow{ \ref{lem: gpsi-dim-elim}} R^d_i - R^d_j > 2 \bet{r} + \epsilon > - \alpha.
        \end{align*}
        (iii) Dimension $d$ was eliminated for one of the groups at round $r'$ for the first time and $\rhdj{r'}$ was larger, then 
        \begin{align*}
            \rhdj{r'} - \rhdi{r'} \geq 4 \bet{r'} + \epsilon \xLongrightarrow{ \ref{lem: gpsi-dim-elim}} \rhdj{r} - \rhdi{r} > 2 \bet{r} + \epsilon,
        \end{align*}
        which is a contradiction with $\rhdi{r} - \rhdj{r} \geq 2 \bet{r} - \alpha$ and proves the statement. 

        For \eqref{eq: prop-1}, note that
        \begin{align*}
            \m{\rhi{r}}{\rhj{r}} \geq 2 \bet{r} &\Longrightarrow \forall d : \rhdj{r} - \rhdi{r} \geq 2 \bet{r} \\
            &\xLongrightarrow{\eqref{eq: prop-in-1}} \forall d: R^d_j - R^d_i > 0 \Longrightarrow \m{\br_i}{\br_j} > 0.
        \end{align*}

        For \eqref{eq: prop-2}, we have
        \begin{align*}
            \Meps{\rhi{r}}{\rhj{r}} \geq 2 \bet{r} &\Longrightarrow \exists d: \rhdi{r} + \epsilon \geq \rhdj{r} + 2 \bet{r} \\
            &\xLongrightarrow{\eqref{eq: prop-in-1}} R^d_i - R^d_j > - \epsilon \Longrightarrow \Meps{\br_i}{\br_j} > 0.  
        \end{align*}

        To prove \eqref{eq: prop-new}, we have
        \begin{align*}
            \m{\rhi{r}}{\rhj{r}} \leq 2 \bet{r} \Longrightarrow \exists d: \rhdj{r} - \rhdi{r} \leq 2 \bet{r}.
        \end{align*}
        There are three different cases: (i) $d$ is active for both $\G_i$ and $\G_j$ at round $r$, then by Lemma \ref{lem: gpsi-arm-elim}
        \begin{align*}
            R^d_j - R^d_i < \rhdj{r} - \rhdi{r} - 2 \bet{r} \leq 0 \leq 4 \bet{r} \Longrightarrow \m{\br_i}{\br_j} < 4 \bet{r}. 
        \end{align*}
        (ii) $d$ was eliminated at $r'$ when $\rhdi{r'}$ was larger, then
        \begin{align*}
            \rhdi{r'} - \rhdj{r'} \geq 4 \bet{r'} + \epsilon \xLongrightarrow{ \ref{lem: gpsi-dim-elim}} R^d_i - R^d_j > 2 \bet{r} + \epsilon,
        \end{align*}
        which contradicts with $\br_i \prec \br_j$. 
        
        (iii) $d$ was eliminated at $r'$ and $\rhdj{r'}$ was larger, then
        \begin{align*}
            \rhdj{r'} - \rhdi{r'} \geq 4 \bet{r'} + \epsilon \xLongrightarrow{ \ref{lem: gpsi-dim-elim}} \rhdj{r} - \rhdi{r} > 2 \bet{r} + \epsilon,
        \end{align*}
        which contradicts with $\rhdj{r} - \rhdi{r} \leq 2 \bet{r}$ and proves the statement. 

        For \eqref{eq: prop-4}, note that
        \begin{align*}
            \Meps{\rhi{r}}{\rhj{r}} \leq 2 \bet{r} \Longrightarrow \forall d: \rhdj{r} + 2 \bet{r} \geq \rhdi{r} + \epsilon,
        \end{align*}

        and again for each dimension $d$, consider three cases: (i) $d$ is active for both $\G_i$ and $\G_j$, then by Lemma \ref{lem: gpsi-arm-elim}
        \begin{align*}
            R^d_j - R^d_i > \rhdj{r} - \rhdi{r} - 2 \bet{r} \geq - 4 \bet{r}. 
        \end{align*}

        (ii) $d$ was eliminated at $r'$ and $\rhdi{r'}$ was larger, then
        \begin{align*}
            \rhdi{r'} - \rhdj{r'} \geq 4 \bet{r'} + \epsilon \xLongrightarrow{ \ref{lem: gpsi-dim-elim}} \rhdi{r} - \rhdj{r} > 2 \bet{r} + \epsilon,
        \end{align*}
        which is a contradiction according to $\rhdj{r} + 2 \bet{r} \geq \rhdi{r} + \epsilon$. 

        (iii) $d$ was eliminated at $r'$ and $\rhdj{r'}$ was larger, then
        \begin{align*}
            \rhdj{r'} - \rhdi{r'} \geq 4 \bet{r'} + \epsilon \xLongrightarrow{ \ref{lem: gpsi-dim-elim}} R^d_j - R^d_i > 2 \bet{r} + \epsilon > - 4 \bet{r}.
        \end{align*}
        So for each $d$, we have $R^d_j - R^d_i > -4 \bet{r}$ which implies $\M{\br_i}{\br_j} < 4 \bet{r}$.
        
    \end{proof}

    Using these properties, we now propose the following two lemmas, which together establish the $(\epsilon, \delta)$-correctness of the TE algorithm.

    \begin{lemma} \label{lem: gpsi-correctness-1}
        At each round $r$, if group $\G_i$ is eliminated during the group rejection phase, then $i \notin \G^*$.
    \end{lemma}
    \begin{proof}
        If group $\G_i$ is eliminated, then there exist a $j \in [N]$ such that
        \begin{align*}
            \m{\rhi{r}}{\rhj{r}} \geq 2 \bet{r} \xLongrightarrow{\eqref{eq: prop-1}} \m{\br_i}{\br_j} > 0 \Longrightarrow i \notin \G^*.
        \end{align*}
    \end{proof}

    \begin{lemma} \label{lem: gpsi-correctness-2}
        At each round $r$, for two groups $i, j \in [N]$ with $\br_i \prec \br_j$ and $\Delta_i = \m{\br_i}{\br_j} > \epsilon$, if $i \in G$, then $j \in G$, moreover $i \notin P$ and $j \notin P$.
    \end{lemma}
    \begin{proof}
        We prove this lemma by induction on $r$. In the first round, all the groups are active and the statement holds. Assume on round $r$, $i,j \in G$. Note that $\m{\br_i}{\br_j} > \epsilon$ implies $\Meps{\br_i}{\br_j} = 0$. On the other hand, we have  
        \begin{align*}
            \text{if } \Meps{\rhi{r}}{\rhj{r}} \geq 2 \bet{r} \xLongrightarrow{\eqref{eq: prop-2}} \Meps{\br_i}{\br_j} > 0,
        \end{align*}  
        which implies that $\Meps{\rhi{r}}{\rhj{r}} < 2 \bet{r}$. This inequality ensures that $i \notin P_1$ and $j \notin P_2$, meaning that $j$ remains in $G$ for the next round, thereby completing the induction step.
        
    \end{proof}

    Note that Lemma $\ref{lem: gpsi-correctness-1}$ establishes that, during the group rejection phase of each round, no optimal group is eliminated. This ensures that $\G^* \subseteq G \cup P$ always holds. Additionally, Lemma $\ref{lem: gpsi-correctness-2}$ proves that for non-optimal groups $i$ with $\Delta_i > \epsilon$, $i$ never enters $P$. These two results together guarantee that, on event $\E$, the algorithm produces a correct output, implying that Algorithm \ref{algo: te} is $(\epsilon, \delta)$-PAC.  

    We now aim to upper bound the number of samples collected from each arm during the learning process. The following lemma establishes that the algorithm will terminate when $\bet{r}$ becomes small.

    \begin{lemma} \label{lem: gspi-sample-1}
        If at round $r$ of the algorithm, $\bet{r} < \frac{\epsilon}{4}$, then $P_2 = P_1 = G_1(r)$ and the algorithm terminates. 
    \end{lemma}
    \begin{proof}
        Note that for each pair of active groups $i,j \in G_1(r)$, we have 
        \begin{align*}
            &\m{\rhi{r}}{\rhj{r}} \leq 2 \bet{r} \Longrightarrow \exists d: \rhdj{r} - \rhdi{r} \leq 2 \bet{r} \\ \xLongrightarrow{\epsilon > 4 \bet{r}} &\rhdj{r} + 2 \bet{r} < \rhdi{r} + \epsilon \Longrightarrow \Meps{\rhi{r}}{\rhj{r}} > 2 \bet{r}.
        \end{align*}

        When this holds for each pair of groups, implies $P_1 = G_1(r)$ and $P_2 = P_1$, and makes $G$ empty for the next round.
    \end{proof}

    The following three lemmas demonstrate that each group $\G_i$ will be eliminated before $\bet{r}$ becomes significantly smaller than the gap $\Delta_i$.

    \begin{lemma} \label{lem: gpsi-sample-2}
        At each round $r$, if $i \notin \G^*$ with $\Delta_i > \epsilon$ and $i \in G_1(r)$, then $\bet{r} > \frac{\Delta_i}{4}$.
    \end{lemma}
    \begin{proof}
        Assume $j \in \G^*$ is such that $\Delta_i = \m{\br_i}{\br_j}$. Based on Lemma \ref{lem: gpsi-correctness-2}, $j$ is in $G$. Then we have
        \begin{align*}
            i \in G_1(r) \Longrightarrow \m{\rhi{r}}{\rhj{r}} \leq 2 \bet{r} \xLongrightarrow{\eqref{eq: prop-new}} \Delta_i = \m{\br_i}{\br_j} < 4 \bet{r}.
        \end{align*}
    \end{proof}

    \begin{lemma} \label{lem: gpsi-sample-3}
        At each round $r$, if $i \in \G^*$ and $i \in G_1(r) \setminus P_1$, then $\bet{r} > \frac{\Delta_i}{4}$.
    \end{lemma}
    \begin{proof}
        $i \in G_1(r) \setminus P_1$ implies that
        \begin{align*}
            \exists j \in G_1(r), j \neq i: \Meps{\rhi{r}}{\rhj{r}} \leq 2 \bet{r} \xLongrightarrow{\eqref{eq: prop-4}} \M{\br_i}{\br_j} < 4 \bet{r}. 
        \end{align*}
        Now consider two cases: (i) $j \in \G^*$. In this case, by definition of $\Delta_i$, we have $\Delta_i \leq \M{\br_i}{\br_j} < 4 \bet{r}$. (ii) $j \notin \G^*$. In this case, there exists a $j' \in \G^*$ such that $\br_j \prec \br_{j'}$. Then we have $\Delta_i \leq \M{\br_i}{\br_{j'}} \leq \M{\br_i}{\br_j} + \M{\br_j}{\br_{j'}} = \M{\br_i}{\br_j} < 4 \bet{r}$. Note that for the second inequality we used the fact that for any vectors $\mathbf{u}, \mathbf{v}, \mathbf{s}$, $\M{\mathbf{u}}{\mathbf{s}} \leq \M{\mathbf{u}}{\mathbf{v}} + \M{\mathbf{v}}{\mathbf{s}}$. 
    \end{proof}

    \begin{lemma} \label{lem: gpsi-sample-4}
        At each round $r$, if $i \in \G^*$, $i \in P_1 \setminus P_2$, and $\bet{r} > \frac{\epsilon}{4}$, then $\bet{r} > \frac{\Delta_i}{12}$.
    \end{lemma}
    \begin{proof}
        $i \in P_1 \setminus P_2$ implies that
        \begin{align*}
             \exists j \in G_1(r) \setminus P_1, j \neq i: \Meps{\rhj{r}}{\rhi{r}} \leq 2 \bet{r} \xLongrightarrow{\eqref{eq: prop-4}} \M{\br_j}{\br_i} < 4 \bet{r}. 
        \end{align*}
        Now we show that for group $\G_j$, we have $\Delta_j \leq 4 \bet{r}$. To prove this consider three cases: (i)$\Delta_j \leq \epsilon$. Then $\Delta_j \leq 4 \bet{r}$ and the proof is complete. (ii) $\Delta_j > \epsilon$ and $j \in \G^*$. In this case, $j \notin P_1$ implies
        \begin{align*}
            \exists l \in G_1(r),l \neq j: \Meps{\rhj{r}}{\hat{\mathbf{R}}_l(r)} < 2 \bet{r} \xLongrightarrow{\eqref{eq: prop-4}} \M{\br_j}{\br_l} < 4 \bet{r},
        \end{align*}
        and exactly like the proof of the previous lemma, we have $\Delta_j \leq 4 \bet{r}$. (iii) $\Delta_j > \epsilon$ and $j \notin \G^*$. In this case, by Lemma \ref{lem: gpsi-correctness-2} there exists $j' \in G_1(r)$ such that $j' \in \G^*$ with $\Delta_j = \m{\br_j}{\br_{j'}} > \epsilon$. Then 
        \begin{align*}
            j \in G_1(r) \Longrightarrow \m{\rhj{r}}{\hat{\mathbf{R}}_{j'}(r)} \leq 2 \bet{r} \xLongrightarrow{\eqref{eq: prop-new}} \Delta_j = \m{\br_j}{\br_{j'}} < 4 \bet{r}.
        \end{align*}
        Finally, by definition $\Delta_i \leq \M{\br_j}{\br_i} + 2 \Delta_j < 12 \bet{t}$. 
    \end{proof}

    The following lemma shows that each arm $(i,j)$ will be eliminated before $\bet{r}$ becomes significantly smaller than $\Delta_{i,j}$.

    \begin{lemma} \label{lem: gpsi-sample-arm}
        At each round $r$, if group $\G_i$ is active and arm $(i,j)$ is active in at least one active dimension $d$, then $\bet{r} > \frac{\Delta_{i,j}}{4}$.
    \end{lemma}
    \begin{proof}
        As the dimension $d$ is active, then by Lemma \ref{lem: gpsi-best-arms}, arm $(i, i^*_d)$ is also active. As arm $(i,j)$ is active in dimension $d$, we have 
        \begin{equation*}
            \mu^d_{i,j} > \muh^d_{i,j}(r) - \bet{r} > \max_{j \in \A_{i,d}} \muh^d_{i,j}(r) - 3 \bet{r} \geq \muh^d_{i, i^*_d}(r) - 3 \bet{r} > R_i^d - 4 \bet{r},
        \end{equation*}
        which shows that $\Delta_{i,j} = \m{\bmu_{i,j}}{\br_i} = \min_{d'} R_i^{d'} - \mu^{d'}_{i,j} \leq R_i^{d} - \mu^d_{i,j} < 4 \bet{r}.$
    \end{proof}

    By Lemmas \ref{lem: gspi-sample-1}, \ref{lem: gpsi-sample-2}, \ref{lem: gpsi-sample-3}, \ref{lem: gpsi-sample-4}, and \ref{lem: gpsi-sample-arm}, for each group $\G_i$, if $\bet{r} \leq \max \left( \frac{\Delta_i}{12}, \frac{\epsilon}{4} \right)$, then $\G_i$ is eliminated at round $r$. Similarly, for each arm $(i,j)$, if $\bet{r} \leq \frac{\Delta_{i,j}}{4}$, arm $(i,j)$ is eliminated at round $r$.  

    Combining this result with Lemma \ref{lem: n-beta}, we conclude that there exists a universal constant $C$ such that, if event $\E$ holds, the number of samples collected from arm $(i,j)$ is at most  
    $$
    \frac{C}{\tilde{\Delta}_{i,j}^2} \log\left( \frac{NKD}{\delta \tilde{\Delta}_{i,j}} \right),
    $$  
    where $\tilde{\Delta}_{i,j} = \max\left( \Delta_i, \Delta_{i,j}, \epsilon \right)$. This completes the proof.  

    Next, we show that the correctness of the algorithm and the same sample complexity upper bound holds even without performing the dimension elimination phase. Consequently, the theoretical upper bound on sample complexity remains valid. This suggests that in most instances where dimension elimination occurs, the algorithm can achieve better performance than this upper bound.  

    To establish this, note that if the TE algorithm omits the dimension elimination phase, then under the good event $\E$, at each round $r$, Lemma \ref{lem: gpsi-best-arms} ensures that arms $(i, i^*_d)$ remain active for each active group $\G_i$ and dimension $d$. Furthermore, Lemma \ref{lem: gpsi-arm-elim} still holds, guaranteeing that  
    $$  
    \lvert \rhdi{r} - R^d_i \rvert < \bet{r}  
    $$  
    for each group $\G_i$. Then, for each pair of groups $\G_i$ and $\G_j$, each dimension $d$, and each round $r$, we have:  
    \begin{align*}  
        \bigg\lvert \m{\rhi{r}}{\rhj{r}} - \m{\br_i}{\br_j} \bigg\rvert < 2 \bet{r}, \\  
        \bigg\lvert \Meps{\rhi{r}}{\rhj{r}} - \Meps{\br_i}{\br_j} \bigg\rvert < 2 \bet{r}.
    \end{align*}  
    
    Using these two inequalities, all the properties in Lemma \ref{lem: gpsi-properties} remain valid, which serve as the foundational ingredients for the remainder of the proof.

\end{proof}

\subsection{Proof of Theorem \ref{thm : gpsi lower}} 

\gpsiLower*

\begin{proof}

To prove the lower bound, we employ the well-known change of measure argument and information-theoretic inequalities.  

Unlike in single-objective bandits, where an instance is fully characterized by the arm means vector $\bmu$, the multi-objective setting introduces an additional complexity: the dependencies among different reward dimensions of each arm. The proposed algorithm is designed to work under any type of dependency, and the derived bound remains valid in all cases. To construct hard instances for the lower bound, we adopt a worst-case strategy concerning these dependencies. For this, we introduce the following definition:  

\begin{definition} [Fully-Dependent Arm] \label{def: fully-dependent}  
Consider a random vector $X = (X_1, X_2, \ldots, X_D)$ with mean vector $\bmu \in \mathbb{R}^D$. We call $X$ \textit{fully-dependent} if  
\begin{align*}  
    \forall d > 1 : X_d = X_1 + (\mu^d - \mu_1).  
\end{align*}  
This condition implies that all dimensions $d > 1$ are completely determined by the value of $X_1$. An arm is said to be fully-dependent if its reward vector is fully-dependent.  
\end{definition}

The following lemma shows an important property of fully-dependent random vectors that we utilize in the proof.

\begin{lemma} \label{lem: fully-dependent-kl}
    Assume $X$ and $Y$ are both fully-dependent random vectors with means $\bmu_X$ and $\bmu_Y$ such that $X_1$ and $Y_1$ are standard Gaussian variables and there exists a constant $\alpha$ satisfying $\bmu_X - \bmu_Y = \alpha \mathbbm{1}$, then
    \begin{align*}
        \kl{\pr_X, \pr_Y} = \frac{\alpha^2}{2}, 
    \end{align*}
    where KL denotes the Kullback-Leibler divergence.
\end{lemma}

\begin{proof}
    Since $X$ and $Y$ are fully-dependent, for every $d\geq 2$ we have
    \[
        X_d = X_1 + (\mu_X^d - \mu_X^1)
        \quad \text{and} \quad
        Y_d = Y_1 + (\mu_Y^d - \mu_Y^1).
    \]
    Thus, all coordinates of $X$ (and similarly of $Y$) are determined by the value of $X_1$ (or $Y_1$). In other words, the distributions $\pr_X$ and $\pr_Y$ are supported on the one-dimensional affine subspace parameterized by the first coordinate.

    Then by the definition of KL divergence we have
    \[
        \kl{\pr_X, \pr_Y} = \kl{\mathcal{N}(\mu_X^1,1),\mathcal{N}(\mu_Y^1,1)}.
    \]
    For one-dimensional Gaussian distributions with unit variance, the KL divergence is given by
    \[
        \kl{\mathcal{N}(\mu_X^1,1),\mathcal{N}(\mu_Y^1,1)} = \frac{(\mu_X^1-\mu_Y^1)^2}{2}.
    \]
    Finally, since $\bmu_X - \bmu_Y = \alpha\,\mathbbm{1}$, we have $\mu_X^1-\mu_Y^1 = \alpha$, and therefore,
    \[
        \kl{\pr_X, \pr_Y} = \frac{\alpha^2}{2}.
    \]
\end{proof}

Now we design a class of instances such that the lower bound stated in the theorem applies to them. 
Consider instance $\V$ with arm means tensor $\bmu \in (0,1)^{N \times K \times D}$ and all arms being fully-dependent with $1$-Gaussian distribution for their first reward dimension. The efficiency vectors of groups are defined as follows
\begin{align*}
    \br_1 \triangleq (a_1, a_2, a_3, a_3, \ldots, a_3), \\
    \br_2 \triangleq (a_2, a_1, a_3, a_3, \ldots, a_3),
\end{align*}
where $a_1 > a_2$ and $a_3 > \epsilon$ are real numbers in $(0,1)$, and for each group $\G_i$ with $i > 2$,
\begin{align} \label{eq: gpsi LB non-opt eff}
    \br_i \triangleq (a_4, a_4, a_5, a_5, \ldots, a_5),
\end{align}
where $a_4$ and $a_5$ are numbers in $(0,1)$ satisfying $a_3 - a_5 > a_2 - a_4  > \epsilon$, $a_1 - a_2 < 2(a_2 - a_4)$, $a_2 > 2 a_4$, and $a_1 < 2a_2$. 
As an example for $\epsilon < 0.1$, consider $a_1 = 0.9, a_2 =0.7, a_3 = 0.7, a_4 = 0.3, a_5 = 0.1$. 

Moreover, let the reward mean vectors of arms in groups $1,2$ satisfy 
\begin{align} \label{eq: gpsi-lower-mu-optimal}
    \forall j : \begin{cases}
        \mu^1_{1,j}, \mu^2_{1,j} > 2a_2 - a_1 , \\
        \mu^1_{2,j}, \mu^2_{2,j} > 2a_2 - a_1,
    \end{cases}
\end{align}

and there is no condition on the reward mean vectors of the arms in non-optimal groups (only the efficiency vector should satisfy \eqref{eq: gpsi LB non-opt eff}). 

Note that these conditions imply that $\G^* = \{1,2\}$, and by definition, the group gaps $\Delta_i$ are as follows:
\begin{align*}
    &\Delta_1 = \Delta_2 = a_1 - a_2, \\
    \forall i > 2: &\Delta_i = a_2 - a_4 > \epsilon,
\end{align*}
which implies $\G^*_{\epsilon} = \{1,2\}$.

Consider an $(\epsilon, \delta)$-PAC algorithm $(P_t, \tau, \hat{i}_{\tau})$ interacting with $\V$ to solve GPSI problem. 

For each arm $(k,l)$ with $k > 2$, we construct an alternative instance $\V(k,l)$ with all arms remaining fully-dependent with $1$-Gaussian reward distributions and the arm means tensor $\bmu(k,l)$ defined as

\begin{align*}
        \bmu(k,l)_{i,j} \triangleq \begin{cases}
            \bmu_{i,j} + (a_2 - \min(\mu^1_{i,j}, \mu^2_{i,j})) & \quad \text{if } (i,j) = (k,l), \\
            \bmu_{i,j} & \quad \text{otherwise.}
        \end{cases}
\end{align*} 

Now, for each $i \in [N]$, define the event $\E_i \triangleq \{ i \in \hat{i}_{\tau}\}$ which represents the event that the algorithm includes group $\G_i$ in the estimated set of Pareto optimal groups. Note that for each $k > 2$, $\Delta_k > \epsilon$, then by correctness of the algorithm, we have $\pr_{\V} (\E_k) < \delta$. On the other hand, in $\V(k,l)$, $\G_k$ is Pareto optimal which implies $\pr_{\V(k,l)} (\E_k) \geq 1 - \delta$.

As the stopping time of the algorithm is almost-surely finite, Lemma $1$ in \cite{lower-kaufmann2016complexity} implies 

\begin{align*}
        \sumij \mathbb{E}_{\V}[N_{i,j}(\tau)] \kl{\nu_{i,j}, \nu_{i,j}(k,l)} \geq d_B\left( \pr_{\V}(\E_k), \pr_{\V(k,l)}(\E_k) \right),
\end{align*}  
    
where $N_{i,j}(\tau)$ denotes the number of samples that the algorithm takes from arm $(i,j)$, $\kl{\nu_{i,j}, \nu_{i,j}(k,l)}$ denotes the Kullback-Leibler divergence between the reward distributions of arm $(i,j)$ in $\V$ and $\V(k,l)$, and $d_B(x,y) = x \log(x/y) + (1-x) \log((1-x)/(1-y))$ is the binary relative entropy.  
    
Since $\V$ and $\V(k,l)$ differ only in the reward distribution of arm $(k,l)$ and $d_B(\delta, 1 - \delta) \geq \log(1/2.4\delta)$, we obtain  
    
\begin{align*}
    \mathbb{E}_{\V}[N_{k,l}(\tau)] \geq \frac{1}{\kl{\nu_{k,l}, \nu_{k,l}(k,l)}} \log\left( \frac{1}{2.4\delta} \right).
\end{align*}  
    
By Lemma \eqref{lem: fully-dependent-kl}, we further obtain  

\begin{align*}
    \mathbb{E}_{\V}[N_{k,l}(\tau)] \geq \frac{2} {(a_2 - \min(\mu^1_{k,l}, \mu^2_{k,l}))^2} \log\left( \frac{1}{2.4\delta} \right).
\end{align*} 

Now, by the conditions provided for the mean reward vectors, we have
\begin{align*}
    &a_2 - \min(\mu^1_{k,l}, \mu^2_{k,l}) < a_2 < 2 (a_2 - a_4) = 2 \Delta_k,  \\ 
    & \mu^d_{k,l} > 0 \Longrightarrow \Delta_{k,l} < a_4 < a_2 - a_4 = \Delta_k, \\
    & \epsilon < a_2 - a_4 = \Delta_k.
\end{align*}
This shows that $\max(\Delta_k, \Delta_{k,l}, \epsilon) = \Delta_k$, then $a_2 - \min(\mu^1_{k,l}, \mu^2_{k,l}) < 2 \max(\Delta_k, \Delta_{k,l}, \epsilon)$ which implies 
\begin{align*}
    \mathbb{E}_{\V}[N_{k,l}(\tau)] \geq \frac{1/2} {\max(\Delta_k, \Delta_{k,l}, \epsilon)^2} \log\left( \frac{1}{2.4\delta} \right).
\end{align*}

For arms $(1,l)$, the alternative instance again has fully-dependent arms with $1$-Gaussian reward distributions and the arm means tensor $\bmu(1,l)$ defined as
\begin{align*}
        \bmu(1,l)_{i,j} \triangleq \begin{cases}
            \bmu_{i,j} + (a_1 - \min(\mu^1_{i,j}, \mu^2_{i,j}) + \epsilon + \alpha) & \quad \text{if } (i,j) = (1,l), \\
            \bmu_{i,j} & \quad \text{otherwise,}
        \end{cases}
\end{align*} 
where $\alpha$ is a small positive number. Note that as $a_1 < 1$, for small values $\epsilon$, the entries of $\bmu(1,l)$ remain in the interval $[0,1]$.

In this case, since $\G_1 , \G_2 \in \G^*(\bmu), \pr_{\V}(\E_2) \geq 1 - \delta$. But for each $l \in [K]$, in instance $\V(1,l)$, the group $\G_2$ is dominated by $\G_1$, and the gap is $\epsilon + \alpha$ which is strictly more than $\epsilon$, so $\G_2$ should not be included in the final guess of the algorithm. It shows that $\forall l \in [K]: \pr_{\V(1,l)} (\E_2) \leq \delta$. 

Then, with exactly the same argument as arms in non-optimal groups, we obtain
\begin{align*}
    \mathbb{E}_{\V}[N_{1,l}(\tau)] \geq \frac{2} {(a_1 - \min(\mu^1_{1,l}, \mu^2_{1,l}) + \epsilon + \alpha)^2} \log\left( \frac{1}{2.4\delta} \right).
\end{align*} 

We further have
\begin{align*}
    &a_1 - \min(\mu^1_{1,l}, \mu^2_{1,l}) \overset{\eqref{eq: gpsi-lower-mu-optimal}}{<} 2 (a_1 - a_2) = 2 \Delta_1, \\ 
    & \Delta_{1,l} \overset{\eqref{eq: gpsi-lower-mu-optimal}}{<} a_1 - a_2 = \Delta_1, \\
    & \epsilon < a_1 - a_2 = \Delta_1,
\end{align*}
which implies $a_1 - \min(\mu^1_{1,l}, \mu^2_{1,l}) + \epsilon < 3 \max(\Delta_1, \Delta_{1,l}, \epsilon)$. Then by limiting $\alpha \rightarrow 0$, we obtain 
\begin{align*}
    \mathbb{E}_{\V}[N_{1,l}(\tau)] \geq \frac{2/9} {\max(\Delta_1, \Delta_{1,l}, \epsilon) ^2} \log\left( \frac{1}{2.4\delta} \right).
\end{align*} 
With a similar argument, for arms $(2,l)$, we obtain 
\begin{align*}
    \mathbb{E}_{\V}[N_{2,l}(\tau)] \geq \frac{2/9} {\max(\Delta_2, \Delta_{2,l}, \epsilon) ^2} \log\left( \frac{1}{2.4\delta} \right),
\end{align*}
which concludes the proof. 

\end{proof}

\section{Proofs of Section \ref{sec: lbgi}} \label{apx: lbgi-proofs}

\subsection{Proof of Theorem \ref{thm : eecb upper}} \label{apx: eecb-upper}

\eecbUpperBound*

\begin{proof}
    
We define the good event $\E$ as 
\begin{equation*}
    \E \triangleq \bigcap_{i \in [N]} \bigcap_{j \in [K]} \bigcap_{d \in [D]} \bigcap_{r \in \mathbb{N}} \left\{ \left| \muh^d_{i,j}(r) - \mu^d_{i,j} \right| < \beta(r, \delta) \right\}.
\end{equation*}

This event is precisely the one introduced in the proof of Theorem \ref{thm : te upper} in Appendix \ref{apx: gpsi}, where we established that $\pr(\E^c) \leq \delta$. We now show that if $\E$ holds, the EECB algorithm correctly identifies the best group and satisfies the sample complexity upper bound stated in the theorem. In the subsequent parts of the proof, we assume that $\E$ holds, implying that some of the results are true by probability $1 - \delta$.

First, we utilize some lemmas from Appendix \ref{apx: gpsi} to demonstrate that under $\E$, some nice properties hold.   

At each round $r$ of the algorithm, if a group $\G_i$ remains active, then by Lemma \ref{lem: gpsi-best-arms}, for each dimension $d$, the arm $(i, i^*_d)$ is also active in dimension $d$, where $i^*_d$ is defined in \eqref{def: istard}. Furthermore, recall that each dimension $d$ is selected by the algorithm to focus on $n^d(r)$ times up to round $r$. This implies that all arms in $\A_d$ (i.e., arms active in dimension $d$) have been pulled at least $n^d(r)$ times, then, by Lemma \ref{lem: gpsi-arm-elim}, we obtain  
\begin{equation*}
    \forall d : \lvert R^d_i - \rhdi{r} \rvert < \betnd{r} \Longrightarrow \lvert \br_i^{\top} \bw - \rhi{r}^{\top}\bw \rvert < \sum_{d \in [D]}  w^d \betnd{r}.
\end{equation*} 

This property implies that  
\begin{align*}
    \forall i \neq \istarmu : \hat{\br}_{\istarmu}(r)^{\top} \bw &> \br_{\istarmu}^{\top} \bw - \sum_{d \in [D]}  w^d \betnd{r} \\ 
    &> \br_{i}^{\top} \bw - \sum_{d \in [D]}  w^d \betnd{r} > \rhi{r}^{\top} \bw - 2 \sum_{d \in [D]}  w^d \betnd{r},
\end{align*}  

which demonstrates that the group $\G_{\istarmu}$ always remains active. This confirms that the EECB algorithm is $\delta$-correct.  

We introduce the following lemma to establish the sample complexity upper bound.

\begin{lemma} \label{lem: eecb-beta}
    At each round $r$, if the arm $(i,j)$ is pulled, then $\bet{n^{\dr}(r)} > \frac18 \left( \frac{\Delta_i \wlone}{Dw^{\dr}} + (R^{\dr}_i - \mu^{\dr}_{i,j}) \right) $. 
\end{lemma}
\begin{proof}
    Since arm $(i,j)$ is pulled, we have two properties: (i) Group $\G_i$ is not eliminated. (ii) Arm $(i,j)$ is not eliminated from $\A_{\dr}$. The first one implies
    \begin{align*}
        &\forall i \neq \istarmu : \rhi{r}^{\top} \bw \geq \hat{\br}_{\istarmu}(r)^{\top} \bw - 2 \sum_{d \in [D]}  w^d \betnd{r} \\
        &\Longrightarrow \br_i^{\top} \bw > \br_{\istarmu}^{\top} \bw - 4 \sum_{d \in [D]}  w^d \betnd{r} \Longrightarrow \sum_{d \in [D]}  w^d \betnd{r} > \frac{\Delta_i \wlone}{4}. 
    \end{align*}
    On the other hand since $\dr = \argmax_{d \in [D]} \betnd{r} w^d$, then $\sum_{d \in [D]}  w^d \betnd{r} \leq D w^{\dr} \bet{n^{\dr}(r)}$ which implies 
    \begin{align*}
        \bet{n^{\dr}(r)} > \frac{\Delta_i \wlone}{4Dw^{\dr}}.
    \end{align*}
    For arms in the optimal group, since $(\istarmu, j)$ is pulled, there exists at least a non-optimal active group $i$ which implies $\bet{n^{\dr}(r)} > \frac{\Delta_i \wlone}{4Dw^{\dr}}$. As $\Delta_{\istarmu} = \min_{i \neq \istarmu} \Delta_i$, then
    \begin{align*}
        \bet{n^{\dr}(r)} > \frac{\Delta_{\istarmu} \wlone}{4Dw^{\dr}}.
    \end{align*}
    
    The property (ii) implies
    \begin{align*}
        \muh^{\dr}_{i,j}(r) \geq \muh^{\dr}_{i, i^*_{\dr}} - 2 \bet{n^{\dr}(r)} &\Longrightarrow \mu^{\dr}_{i,j} > R^{\dr}_i - 4 \bet{n^{\dr}(r)} \\ 
        &\Longrightarrow \bet{n^{\dr}(r)} > \frac{R^{\dr}_i - \mu^{\dr}_{i,j}}{4}.
    \end{align*}

    Combining these two inequalities, we obtain
    \begin{align*}
        \bet{n^{\dr}(r)} > \frac18 \left( \frac{\Delta_i \wlone}{Dw^{\dr}} + (R^{\dr}_i - \mu^{\dr}_{i,j}) \right).
    \end{align*}
\end{proof}

For each arm $(i,j)$ and dimension $d$, let $\Delta'_{i,j}(d) \triangleq \Delta_i \wlone/Dw^{d} + (R^{d}_i - \mu^{d}_{i,j})$. The following lemma shows the upper bound on the number of arm pulls for each arm $(i,j)$.

\begin{lemma} \label{lem: eecb-sample}
    For each arm $(i,j)$, $N_{i,j}(\tau) \leq \max_{d \in [D]} \frac{C \log\left( \frac{NKD}{\delta \Delta'_{i,j}(d)} \right)}{\Delta'_{i,j}(d)^2} + 1$, where $\tau$ is the stopping time of the algorithm and $C$ is a universal constant.
\end{lemma}
\begin{proof}
    At each round $r$, if arm $(i,j)$ is pulled, then $N_{i,j}(r) = n^{\dr}(r)$. By Lemma \ref{lem: eecb-beta}, $\bet{n^{\dr}(r)} > \frac{\Delta'_{i,j}(\dr)}{8}$, then Lemma \ref{lem: n-beta} implies that there exists a constant $C$ such that
    \begin{align*}
        N_{i,j}(r) \leq {\frac{C \log\left( \frac{NKD}{\delta \Delta'_{i,j}(\dr)} \right)}{\Delta'_{i,j}(\dr)^2}}.
    \end{align*}
    As the above inequality holds for every $r$, consider the last round $r_f$ where the algorithm pulls arm $(i,j)$ and then stops (this stopping time is finite under $\E$), then 
    \begin{align*}
        N_{i,j}(\tau) = N_{i,j}(r_f) + 1 \leq {\frac{C \log\left( \frac{NKD}{\delta \Delta'_{i,j}({\text{d}(r_f)})} \right)}{\Delta'_{i,j}({\text{d}(r_f)})^2}} + 1 \leq \max_{d \in [D]} \frac{C \log\left( \frac{NKD}{\delta \Delta'_{i,j}(d)} \right)}{\Delta'_{i,j}(d)^2} + 1. 
    \end{align*}
\end{proof}

To complete the upper bound proof, it remains to show that for each arm $(i,j)$,  

\begin{align*}
    \min_{d \in [D]} \Delta'_{i,j}(d) \geq \frac{\Delta_{i,j}}{D}.
\end{align*}  
This is sufficient because the function  $f(x) \triangleq  \frac{C \log\left( \frac{NKD}{\delta x} \right)}{x} $ is decreasing for $0 < x \leq 1$.  

To establish this, for $i \neq \istarmu$, note that  
\begin{align}
    \Delta'_{i,j}(d) \geq \frac{1}{D} \left( \frac{\Delta_i \wlone}{w^d} + (R^d_i - \mu^d_{i,j}) \right).
\end{align}  

Here, $\frac{\Delta_i \wlone}{w^d} + (R^d_i - \mu^d_{i,j})$ represents the minimum positive value that must be added to $\mu^d_{i,j}$ to make group $\G_i$ the best group. Clearly, this quantity is at most $\Delta_{i,j}$, which is equal to $\alpha_{i,j}$ for arms in non-optimal groups and this value shows the minimum value that must be added to \textit{all} dimensions of $\bmu_{i,j}$ to make $\G_i$ optimal. Clearly, the former is larger.

For arms $(\istarmu, j)$, by definition, $\Delta_{\istarmu, j} = \max \left( \Delta_{\istarmu}, \m{\bmu_{\istarmu, j}}{\br_{\istarmu}}\right) $. Then we need to show 
\begin{align*}
    \min_{d \in [D]} \Delta'_{i,j} \geq  \max \left( \Delta_{\istarmu}, \m{\bmu_{\istarmu, j}}{\br_{\istarmu}}\right).
\end{align*}

For this, note that
\begin{align*}
    &\min_{d \in [D]} R_{\istarmu}^d - \mu_{\istarmu,j}^d = \m{\bmu_{\istarmu, j}}{\br_{\istarmu}}, \\
    &\min_{d \in [D]} \frac{\Delta_{\istarmu} \wlone}{w^d} \geq \frac{\Delta_{\istarmu} \wlone}{\sum_{d \in [D]}w^{d}} = \Delta_{\istarmu}.
\end{align*}

This implies that
\begin{align*}
    \min_{d \in [D]} \Delta'_{i,j} (d) &\geq \frac{1}{D} \left( \min_{d \in [D]} \frac{\Delta_{\istarmu} \wlone}{w^d} + \min_{d \in [D]} R_{\istarmu}^d - \mu_{\istarmu,j}^d \right) \\
    &\geq \frac{1}{D} \left( \Delta_{\istarmu} + \m{\bmu_{\istarmu, j}}{\br_{\istarmu}} \right) \geq \frac{\Delta_{\istarmu,j}}{D},
\end{align*}

which completes the proof.

\end{proof}

\subsection{Proof of Theorem \ref{thm : lbgi lower}}

\lbgiLower*

\begin{proof}

    Consider an instance $\V$ with arm means tensor $\bmu$, where all arms are fully-dependent (Definition \ref{def: fully-dependent}) with $1$-Gaussian reward distributions, and a $\delta$-correct algorithm $(P_t, \tau, \hat{i}_{\tau})$ interacting with $\V$ to solve the LBGI problem with weight vector $\bw$. For each arm $(k,l)$ with $k \neq \istarmu$, we construct an alternative instance $\V(k,l)$ with all arms remaining fully-dependent with $1$-Gaussian reward distributions and arm means tensor $\bmu(k,l)$ defined as follows:  
    
    \begin{align*}
        \bmu_{i,j}(k,l) \triangleq \begin{cases}
            \bmu_{k,l} + \Delta_{k,l} + \alpha &\quad \text{if } (i,j) = (k,l), \\
            \bmu_{i,j}  &\quad \text{otherwise,}
        \end{cases}
    \end{align*}  
    
    where $\alpha$ is a small positive number.  
    
    
    For each $k \neq \istarmu$, define the event $\E_{k} \triangleq \{\hat{i}_{\tau} = k\}$, which represents the algorithm selecting group $\G_k$ as the best group. Since the algorithm is $\delta$-correct, we have $\pr_{\V}(\E_k) \leq \delta$ and $\pr_{\V(k,l)}(\E_k) \geq 1 - \delta$ for each $l \in [K]$.  
    
    As the stopping time of the algorithm is almost surely finite, we apply Lemma 1 of \cite{lower-kaufmann2016complexity}:  
    
    \begin{align*}
        \sumij \mathbb{E}_{\V}[N_{i,j}(\tau)] \kl{\nu_{i,j}, \nu_{i,j}(k,l)} \geq d_B\left( \pr_{\V}(\E_k), \pr_{\V(k,l)}(\E_k) \right),
    \end{align*}  
    
    where $\kl{\nu_{i,j}, \nu_{i,j}(k,l)}$ denotes the Kullback-Leibler divergence between the reward distributions of arm $(i,j)$ in $\V$ and $\V(k,l)$, and $d_B(x,y) = x \log(x/y) + (1-x) \log((1-x)/(1-y))$ is the binary relative entropy.  
    
    Since $\V$ and $\V(k,l)$ differ only in the reward distribution of arm $(k,l)$ and since $d_B(\delta, 1 - \delta) \geq \log(1/2.4\delta)$, we obtain  
    
    \begin{align*}
        \mathbb{E}_{\V}[N_{k,l}(\tau)] \geq \frac{1}{\kl{\nu_{k,l}, \nu_{k,l}(k,l)}} \log\left( \frac{1}{2.4\delta} \right).
    \end{align*}  
    
    By Lemma \eqref{lem: fully-dependent-kl}, we further obtain  
    
    \begin{align*}
        \mathbb{E}_{\V}[N_{k,l}(\tau)] \geq \frac{2} {(\Delta_{k,l} + \alpha)^2} \log\left( \frac{1}{2.4\delta} \right), 
    \end{align*}  
    
    and by taking the limit as $\alpha \rightarrow 0$ 
    
    \begin{align*}
        \mathbb{E}_{\V}[N_{k,l}(\tau)] \geq \frac{2}{\Delta_{k,l}^2} \log\left( \frac{1}{2.4\delta} \right).
    \end{align*}  
    
    To establish the bound for arms in $\istarmu$, we construct an instance $\V'$ with fully-dependent arms with $1$-Gaussian reward distributions and arm means tensor $\bmu'$ defined as  
    
    \begin{align*}
        \bmu'_{i,j} \triangleq \begin{cases}
            \bmu_{i,j} - {\Delta_{\istarmu} + \alpha } & \quad \text{if } i = \istarmu, \\
            \bmu_{i,j} & \quad \text{otherwise,}
        \end{cases}
    \end{align*}  
    
    where $\alpha$ is a small positive constant. By defining the event $\E' \triangleq \{\hat{i}_{\tau} = \istarmu\}$, Lemma 1 in \cite{lower-kaufmann2016complexity} implies  
    
    \begin{align*}
        \sum_{j \in [K]} \mathbb{E}_{\V}[N_{\istarmu, j}(\tau)] \geq \frac{1}{\kl{\nu_{\istarmu, 1}, \nu'_{\istarmu, 1}}} d_B\left( \pr_{\V}(\E'), \pr_{\V'}(\E') \right),
    \end{align*}  
    
    where $\kl{\nu_{\istarmu, 1}, \nu'_{\istarmu, 1}}$ denotes the Kullback-Leibler divergence between the reward distributions of arm $(\istarmu,1)$ in $\V$ and $\V'$. The inequality holds because the modification in means is uniform across all arms in the optimal group. This adjustment makes group $\istarmu$ non-optimal in $\V'$ which means that $\pr_{\V}(\E') \geq 1 - \delta$ and $\pr_{\V'}(\E') \leq \delta$.  
    
    By Lemma \eqref{lem: fully-dependent-kl} and taking the limit as $\alpha \rightarrow 0$, we obtain  
    
    \begin{align*}
        \sum_{j \in [K]} \mathbb{E}_{\V}[N_{\istarmu, j}(\tau)] \geq \frac{ 2}{\Delta_{\istarmu}^2}  \log\left( \frac{1}{2.4\delta} \right), 
    \end{align*}  
    
    which completes the proof.

    \section{Correctness of Baselines} \label{apx: exp-baselines-correctness}  
    
    This section provides a theoretical justification for the correctness of the baseline algorithms used in the experiments for the GPSI problem. The correctness of AGE is established in the proof of Theorem \ref{thm : te upper} in Appendix \ref{apx: gpsi}. The same argument applies to GE, demonstrating that it is also $(\epsilon, \delta)$-PAC.  
    
    To establish the correctness of UniS, we define the good event $\E'$ as follows:  
    \begin{align*}
        \E' \triangleq \bigcap_{i \in [N]} \bigcap_{j \in [K]} \bigcap_{d \in [D]} \left\{ \left| \muh^d_{i,j}(r_0) - \mu^d_{i,j} \right| < \frac{\epsilon}{2} \right\},
    \end{align*}  
    
    where  
    $$
    r_0 = \frac{8}{\epsilon^2}\log\left( \frac{2NKD}{\delta} \right).
    $$  
    
    Applying Hoeffding’s inequality and a union bound over all arms and dimensions, we obtain $\pr(\E'^c) \leq \delta$. Under event $\E'$, for each group $\G_i$ and dimension $d$, we have:  
    
    \begin{align*}
        \lvert \rhdi{r_0} - R^d_i \rvert < \frac{\epsilon}{2},
    \end{align*}  
    
    which further implies  
    \begin{align*}
        \bigg\lvert \m{\rhi{r}}{\rhj{r}} - \m{\br_i}{\br_j} \bigg\rvert < {\epsilon}.
    \end{align*}  
    
    Thus, any optimal group would be recognized as optimal or non-optimal with a gap less than $\epsilon$, and any group $\G_i$ identified as optimal satisfies $\Delta_i < \epsilon$, proving that UniS is $(\epsilon, \delta)$-PAC.

\end{proof}

\end{document}